\documentclass{article}

\usepackage{microtype}
\usepackage{graphicx}
\usepackage{subcaption}
\usepackage{booktabs} 
\usepackage{algorithm}
\usepackage{algorithmic}
\usepackage{graphicx}
\usepackage{overpic}

\usepackage{hyperref}

\usepackage[preprint]{icml2026}

\usepackage{amsmath}
\usepackage{amssymb}
\usepackage{mathtools}
\usepackage{amsthm}

\usepackage[capitalize,noabbrev]{cleveref}

\theoremstyle{plain}
\newtheorem{theorem}{Theorem}[section]

\newtheorem{lemma}[theorem]{Lemma}
\newtheorem{corollary}[theorem]{Corollary}
\theoremstyle{definition}

\newtheorem{assumption}[theorem]{Assumption}
\theoremstyle{remark}

\usepackage[textsize=tiny]{todonotes}

\newcommand{\R}{\mathbb{R}}
\newcommand{\E}{\mathbb{E}}

\newcommand{\Var}{\mathrm{Var}}
\newcommand{\Cov}{\mathrm{Cov}}

\newcommand{\norm}[1]{\left\lVert#1\right\rVert}

\icmltitlerunning{Gaussian Process Bandit Optimization with Machine Learning Predictions}

\begin{document}

\twocolumn[
  \icmltitle{Gaussian Process Bandit Optimization with Machine Learning Predictions\\ and Application to Hypothesis Generation}



  \icmlsetsymbol{equal}{*}

  \begin{icmlauthorlist}
    \icmlauthor{ Xin Jennifer Chen}{equal,yyy}
    \icmlauthor{Yunjin Tong}{equal,sch}
  
  \end{icmlauthorlist}

  \icmlaffiliation{yyy}{Department of Management Science and Engineering, Stanford University, Stanford, USA}
  \icmlaffiliation{sch}{Graduate School of Business, Stanford University, Stanford, USA}

  \icmlcorrespondingauthor{Yunjin Tong}{yjtong@stanford.edu}

  \icmlkeywords{Gaussian Process, Machine Learning}

  \vskip 0.3in
]



\printAffiliationsAndNotice{\icmlEqualContribution (in alphabetical order).}

\begin{abstract}
Many real-world optimization problems involve an expensive ground-truth oracle
(e.g., human evaluation, physical experiments) and a cheap, low-fidelity prediction oracle
(e.g., machine learning models, simulations).
Meanwhile, abundant offline data (e.g., past experiments and predictions) are often
available and can be used to pretrain powerful predictive models, as well as to provide
an informative prior.
We propose Prediction-Augmented Gaussian Process Upper Confidence Bound (PA-GP-UCB),
a novel Bayesian optimization algorithm that leverages both oracles and offline data to achieve provable gains in sample efficiency for the ground-truth oracle queries.
PA-GP-UCB employs a control-variates estimator derived from a joint Gaussian process
posterior to correct prediction bias and reduce uncertainty.
We prove that PA-GP-UCB preserves the standard regret rate of GP-UCB while achieving a
strictly smaller leading constant that is explicitly controlled by prediction quality
and offline data coverage.
Empirically, PA-GP-UCB converges faster than Vanilla GP-UCB and na\"{\i}ve prediction-augmented GP-UCB baselines on synthetic benchmarks and
on a real-world hypothesis evaluation task grounded in human behavioral data, where
predictions are provided by large language models. These results establish PA-GP-UCB as a general and
sample-efficient framework for hypothesis generation under expensive feedback.
 
\end{abstract}


\section{Introduction}
Many real-world optimization problems, ranging from scientific discovery and engineering design to creative content generation, require exploring high-dimensional, continuous spaces to find inputs that maximize an unknown objective function \cite{frazier2018tutorialbayesianoptimization,shahriari2015taking,NIPS2012_05311655}. It is common in these settings to have the presence of two complementary information sources: an expensive, high-fidelity oracle (e.g., human evaluation, physical experiment, or high-accuracy simulation) that provides accurate but costly feedback, and a cheap, low-fidelity prediction (e.g., a machine learning model, coarse simulation, large language models (LLMs) or analytical approximation) that gives fast but potentially biased estimates \cite{forrester2009recent,  kandasamy2016gaussian,doi:10.1137/16M1082469}. The low-fidelity prediction oracle is especially relevant given the rise in the capability of LLMs \cite{ea2020language,kaplan2020scaling}.   

Meanwhile, in many applications, abundant offline data, collected from past experiments, historical records, or under different experimental conditions, can be leveraged to pretrain predictive models and construct informative priors \cite{bai2023transfer}. This has renewed interest in hybrid pipelines where models guide search but high-fidelity evaluation remains the decision bottleneck. More recently, researchers have begun using AI systems to prioritize hypotheses for follow-up study \cite{ludwig2024machine}, and labs report using LLMs as lightweight prediction oracles in scientific workflows \cite{anthropic2026accelerating}. On the other hand, most practitioners still rely on high-fidelity oracles to ensure the reliability of their optimization results \cite{diessner2022investigating, brochu2010tutorial}. The key question is how to optimally combine online expensive feedback with both an inexpensive prediction and existing offline data to efficiently navigate the search space and converge to the optimum with minimal use of the expensive oracle \cite{wang2024pre, foumani2023multi}. This problem fundamentally calls for a Bayesian sequential decision-making framework that can formalize prior knowledge, propagate model uncertainty, and support provable exploration–exploitation trade-offs under costly evaluations.

Gaussian process (GP) bandit algorithms, such as Gaussian Process Upper Confidence Bound (Vanilla GP-UCB) \cite{Srinivas2010}, provide such a framework for global optimization of unknown functions in continuous domains \cite{williams2006gaussian}. By modeling the objective with a GP and sequentially selecting query points that balance exploration and exploitation, Vanilla GP-UCB achieves sublinear cumulative regret with high probability. However, Vanilla GP-UCB relies solely on the expensive online oracle and does not leverage cheap prediction information or offline data that may be available in many applications. This limitation motivates the development of methods that can incorporate prediction data and offline observations to accelerate optimization while maintaining theoretical guarantees.

Meanwhile, a complementary line of work explores how to leverage biased but informative
predictions in statistical inference and decision-making.
Prediction-Powered Inference (PPI) \cite{angelopoulos2023prediction} provides a general
framework for incorporating offline-trained predictive models via bias correction using
limited ground-truth data, but it is designed for static inference rather than sequential
learning. A follow-up work by \cite{ji2025multiarmedbanditsmachinelearninggenerated} extends this
idea to sequential decision-making, where the Machine Learning-Assisted Upper Confidence Bound (MLA-UCB) algorithm uses machine-learned predictions to accelerate exploration
in discrete multi-armed bandits. However, these approaches are restricted to discrete action spaces and do not exploit
structural relationships between arms, nor do they address how to jointly combine
predictions, offline data, and expensive online feedback in continuous domains,
such as those arising in hypothesis generation tasks.

We propose Prediction-Augmented Gaussian Process Upper Confidence Bound (PA-GP-UCB),
a novel algorithm that extends the GP-UCB framework to settings with both an expensive
ground-truth oracle and a cheap, low-fidelity prediction oracle, together with offline
data.
To leverage the statistical relationship between these information sources,
PA-GP-UCB models the high-fidelity ground-truth function and the low-fidelity prediction
as correlated tasks under a multi-task Gaussian process prior. The algorithm proceeds in two stages.
In the offline stage, PA-GP-UCB acquires low-cost prediction oracles over a uniform,
space-filling design, using existing offline data or by making queries to the prediction oracle.
In the online stage, the algorithm sequentially queries the expensive oracle while obtaining corresponding low-fidelity predictions at the same query points.
At each iteration, PA-GP-UCB integrates low-fidelity predictions with high-fidelity
observations through a control-variates estimator, enabling variance reduction and bias
correction.
The next query point is then selected by maximizing an upper confidence bound derived
from this fused estimator. In contrast to classical multi-fidelity bandit formulations \cite{Forrester2007Multifidelity,kandasamy2016gaussian},
we treat predictions as effectively free side information rather than alternative query
fidelities.

The PA-GP-UCB method is general and applicable to continuous optimization problems
involving two information sources with different costs and fidelities.
While PA-GP-UCB can naturally leverage offline data when such data are available, it does
not strictly rely on historical datasets; in the absence of prior data, offline
information can instead be generated using predictive sources such as learned models or LLMs. Our main contributions are as follows:
\begin{itemize}
\item We introduce PA-GP-UCB, a Gaussian process bandit algorithm that leverages
expensive and cheap oracles together with offline data via a novel control-variates
estimator, enabling principled variance reduction and bias correction.

\item We provide a theoretical analysis of PA-GP-UCB, showing that it achieves a
cumulative regret bound that improves upon Vanilla GP-UCB by a constant factor
explicitly controlled by the prediction correlation and offline data coverage, and
empirically validate these gains on synthetic Gaussian process benchmarks.

\item We present PA-GP-UCB as a general framework for sample-efficient hypothesis
generation under expensive feedback, and demonstrate its effectiveness on a real-world
hypothesis evaluation task grounded in human behavioral data, where it efficiently
uncovers high-quality hypotheses using predictions from LLMs.
\end{itemize}

PA-GP-UCB provides a novel framework for sample-efficient hypothesis generation in settings with expensive feedback.
In this setting, the expensive oracle corresponds to human evaluation or physical experimentation, while cheap, low-fidelity predictions are provided by LLMs \cite{Zhou_2024,Forrester2007Multifidelity}.
By maximizing the upper confidence bound of a bias-corrected posterior, PA-GP-UCB naturally proposes novel candidate hypotheses while improving the efficiency of hypothesis evaluation, enabling extrapolation beyond the initial set of human-proposed ideas. To our knowledge, it is the first systematic framework that enables principled, integrated hypothesis generation and efficient evaluation by leveraging predictions in embedded continuous hypothesis spaces.
This framework is particularly well suited to scientific hypothesis generation, where hypotheses exhibit meaningful structure, data are limited, and predictive models may be imperfect. Details of this application and the corresponding experimental results are presented in Section~\ref{sec:experiments}. A more detailed discussion of related work is provided
in Appendix~\ref{related}.

\section{Problem Statement and Preliminaries}


 Consider the problem of  optimizing an unknown true reward function  $f\equiv f^{\mathrm{true}}:\mathcal{X} \mapsto \mathbb{R}$,  where $\mathcal{X} \subseteq[0,r]^d$ is a compact and convex set. We have two noisy function value oracles $y^{\mathrm{i}}:\mathcal{X}\mapsto \mathbb{R}$ for $i
\in\{\mathrm{true},\mathrm{ML}\}$. The ground-truth oracle $y(x)\equiv y^{\mathrm{true}} (x)=f(x)+\varepsilon$ is a noisy observation of the true reward function, with the independent observation noise $\varepsilon \sim \mathcal{N}(0,\eta^2)$. The machine-learning-prediction oracle (which we refer to as the ML oracle or the prediction oracle) outputs noisy observations of $f^{\mathrm{ML}}:\mathcal X \mapsto \R$, which is a potentially biased machine learning prediction of $f$. In particular, $y^{\mathrm{ML}}(x) = f^{\mathrm{ML}}(x)+\varepsilon_{\text{ML}}$ with $\varepsilon_{\text{ML}} \sim \mathcal{N}(0,\eta_{\mathrm{ML}}^2)$, with $\varepsilon_{\text{ML}} $  and $\varepsilon$ independent of each other and of $f^{\mathrm{ML}}$.

We consider the two-stage offline-online setting. In the offline stage, only the prediction oracle can be accessed  a finite number of times. We refer to the offline dataset as  $\mathcal{D}^{\text{off}}=\left\{\left(x_i,y^{\mathrm{ML}}(x_i)\right) \right\}_{i\in [n_{\text{off}}]}$\footnote{We adopt the common notation that $[k]=\{1,2,\ldots,k\}$.}, and the algorithm can decide $n_{\text{off}}$ how many points and which points to query. The online stage proceeds sequentially with a finite horizon $T$ -- at each round $t$, when querying a point, both oracles  are observed by the algorithm, and the algorithm picks the next point to query, repeatedly for $T$ rounds. Let $\mathcal{D}_t = \left\{\left(x_i,y(x_i),y^{\mathrm{ML}}(x_i)\right)\right\}_{i\in [t]}$ denote the online dataset up to round $t$. This setting mirrors common practical scenarios where the true reward function is expensive to evaluate and must therefore be optimized sequentially, while a cheap prediction oracle can be queried at scale ahead of time. Observing both oracles online is natural because once an expensive query is made at a point, the marginal cost of computing the prediction at the same point is negligible, enabling the algorithm to exploit cross-source correlation throughout learning.

\paragraph{Objective} The goal of this paper is to maximize cumulative reward $\sum_{t=1}^T f(x_t)$, which can be interpreted as aiming to perform as well as $x_* = \arg\max_{x\in \mathcal X}f(x)$.   Cumulative regret  is a natural performance metric for this goal.   After an algorithm is run for $T$ rounds, the cumulative regret is $R_T \triangleq \sum_{t\in[T]} \left[f(x_*)-f(x_t)\right]  = \sum_{t\in[T]}r_t$, and we define $r_t = f(x_*)-f(x_t)$ as the instantaneous regret in round $t$.

\paragraph{Gaussian Process Model} We model the true and ML-predicted reward functions jointly as a sample from a two–output Gaussian process:
\[
\begin{bmatrix}
f\\[2pt]
f^{\mathrm{ML}}
\end{bmatrix}
\sim 
\mathcal{GP}\!\left(
\begin{bmatrix}
0\\[2pt]
0

\end{bmatrix},\;
K(x,x')
\otimes 
B
\right),
\qquad 
B =
\begin{bmatrix}
1 & \rho\\[3pt]
\rho & 1
\end{bmatrix}.
\]

It is necessary to impose some structural property on the GP kernel in a compact domain setting, and we consider the following, which is  identical to the classical GP bandits setting in \cite{Srinivas2010}.
\begin{assumption}
\label{assumption:lip} For some constants $a,b>0$,
    \[
    \mathbb{P}\left[ \sup_{x\in \mathcal{X}}\left|\frac{\partial f^{i}}{\partial x_j}\right|>L\right] \leq ae^{-(L/b)^2}, j\in[d],i\in\{\mathrm{true},\mathrm{ML}\}.
    \]
\end{assumption}
If set $L=b\sqrt{\log(da/\delta)}$, then we get with probability greater than $1-\delta$,
\[
|f^{i}(x)-f^{i}(x')|\leq L \norm{x-x'}_1,   \forall x,x'\in \mathcal{X},i\in\{\mathrm{true},\mathrm{ML}\}.
\]
We further restrict to GPs with nonzero and bounded variance $0<K(x,x)\leq 1$.

\paragraph{Notations} To simplify notations, we define the following quantities from the joint posteriors at round $t\in[T]$, $i\in\{\mathrm{true},\mathrm{ML}\}$, \[\mu^{i}_{t}(x)\triangleq \E[f^{i}(x)| \mathcal{D}_{t-1}],\;\mu^{\text{off}}(x)\triangleq\E[f^{\mathrm{ML}}(x)\mid \mathcal{D}^{\text{off}}], \]
\[ \mu_{t}^{\mathrm{ML}, \text{all}}(x)\triangleq\E[f^{\mathrm{ML}}(x)\mid \mathcal{D}^{\text{off}}\cup\mathcal{D}_{t-1}],
\]
\[
(\sigma_t^i(x))^2\triangleq\mathrm{Var}\!\big(f^{i}(x)\mid \mathcal{D}_{t-1}\big),\]
\[(\sigma^{\text{off}}(x))^2\triangleq\mathrm{Var}\!\big(f^{\mathrm{ML}}(x)\mid \mathcal{D}^{\text{off}}\big),
\]
\[(\sigma^{\mathrm{ML},\text{all}}_t(x))^2\triangleq\mathrm{Var}\!\big(f^{\mathrm{ML}}(x)\mid \mathcal{D}^{\text{off}}\cup \mathcal{D}_{t-1}\big),\]
\[ \rho_t(x)\triangleq\frac{\Cov(f(x),f^{\mathrm{ML}}(x)\mid \mathcal{D}_{t-1} )}{\sigma^{\mathrm{true}}_t(x)\sigma^{\mathrm{ML}}_t(x)}.\]

\section{The PA-GP-UCB Algorithm}
\noindent
This section introduces PA-GP-UCB, with details described in Algorithm \ref{alg:gp-mla-ucb}. When applying the PA-GP-UCB algorithm, we maintain two joint GP models at each round $t$, the global $\mathcal{GP}_{\mathrm{all}}$ trained on $\mathcal{D}^{\mathrm{off}} \cup \mathcal{D}_{t-1}$, and  the online $\mathcal{GP}$ trained on only online data $\mathcal{D}_{t-1}$. 

\paragraph{Offline Stage Design} PA-GP-UCB begins with an offline stage, which queries only the prediction oracle on an $\varepsilon$-net \footnote{For \(\mathcal X\subseteq[0,r]^d\) under the \(\ell_\infty\) metric, an \(\varepsilon\)-net can be constructed by a uniform grid: partition each coordinate into intervals of width \(2\varepsilon\) and include the center of each resulting axis-aligned hypercube. Then for any \(x\in\mathcal X\) there exists a grid center \(c\) with \(\|x-c\|_\infty\le \varepsilon\). The net size is at most \(\lceil r/(2\varepsilon)\rceil^d\).} of $\mathcal X$ and collects $N$ repeated observations at each grid center. These observations are then used to perform GP updating on $\mathcal{GP}{\mathrm{all}}$. Repeated queries reduce the effective observation noise from $\eta_{\mathrm{ML}}^2$ to $\eta_{\mathrm{ML}}^2/N$ at grid center points. Then through a high enough coverage by the $\varepsilon$-net, GP interpolation reduces $\sigma_t^{\mathrm{ML},\text{all}}(x)$ for  all $x\in \mathcal X$. Note that the offline stage could be parallelized to finish in $\mathcal O(1)$ round\footnote{We measure round complexity under a parallel query model: in one round we issue up to $Q$ oracle queries non-adaptively, with $Q=\mathcal O (\mathrm{poly}(d,T))$, and receive all responses simultaneously. Under this model, the offline stage has constant round complexity.}. 

We summarize the effect of offline prediction coverage by a uniform ratio bound $(\sigma_t^{\mathrm{ML},\text{all}}(x)/\sigma_t^{\mathrm{ML}}(x))^2\leq R$ for all $x\in\mathcal X$, $t\leq T$. Smaller R means offline data makes the ML-predicted posterior uniformly tighter, and will result in a greater improvement in cumulative regret, described in more detail in Section \ref{sec:regret}.  This  turns offline data design $(\varepsilon,N)$ into a knob that directly controls regret constants. In Lemma \ref{lem:offline_req}, we provide detailed sufficient conditions on the offline procedure parameters $(\varepsilon, N)$ that guarantee the uniform ratio bound to hold for any selected value of $R$ with high probability. Asymptotically\footnote{$\tilde{\mathcal O}$ hides polylogarithmic factors (e.g., powers of $\log T$, $\log d$, etc.).}, the sufficient conditions required for any $R\in(0,1]$ are
\begin{equation}
\label{eqn:suff-cond}
  \varepsilon = \tilde{\mathcal{O}}\left( \sqrt{\frac{R}{T}}\frac{\eta_{\mathrm{ML}}}{d}\right), \text{ and }N = {\Omega}\left(\frac{T}{R}\right).  
\end{equation}
 
The sufficient conditions in (\ref{eqn:suff-cond}) are worst-case and conservative; Sections \ref{sec:numerical} and \ref{sec:experiments} show substantial decrease in cumulative regret even with $\varepsilon$ much larger and N much smaller than required by (\ref{eqn:suff-cond}).


\paragraph{Online Stage Design} To leverage the offline prediction oracle queries in a statistically efficient way, we introduce a control-variates estimator (\ref{eq:mu_mla}) to construct an upper confidence bound for the true reward function used in the online stage: 
\begin{equation}\label{eq:mu_mla}
     \mu^{\mathrm{PA}}_{t}(x) \triangleq  \mu_t^{\mathrm{true}}(x)-\frac{\rho_t(x)\sigma_t^{\mathrm{true}}(x)}{\sigma_t^{\mathrm{ML}}(x)}\big(\mu_t^{\mathrm{ML}}(x)-\mu_{t}^{\mathrm{ML},\text{all}}(x)\big).
\end{equation}
PPI-inspired estimators use a model prediction as a baseline and correct it with a mean-zero residual computed from ground-truth labels \cite{angelopoulos2023prediction}. In our setting, the GP posterior enables an analogous decomposition: the residual $\mu_t^{\mathrm{ML}}(x)-\mu_t^{\mathrm{ML},\text{all}}(x)$ has mean zero conditional on online data, and the GP supplies the optimal (minimum MSE) linear predictor coefficient through $\rho_t(x)\sigma_t^{\mathrm{true}}(x)/\sigma_t^{\mathrm{ML}}(x)$ \cite{anderson2003multivariate}. Unlike MLA-UCB in the discrete-arm setting \cite{ji2025multiarmedbanditsmachinelearninggenerated}, PA-GP-UCB uses a full probabilistic prediction-augmented posterior over a continuous domain, where the offline data reduces uncertainty globally. 



 PA-GP-UCB then selects the next point that maximizes the prediction-augmented upper confidence bound:
\begin{equation}
  \varphi_t(x)\triangleq \mu^{\mathrm{PA}}_{t}(x) + \sqrt{\beta_t}\,\sigma^{\mathrm{PA}}_{t}(x).
\end{equation}
It is shown in Lemma \ref{prop:error_dist} that $(\sigma_t^{\mathrm{PA}}(x))^2$ is equal to
\begin{equation}
\begin{aligned}
(\sigma_t^{\mathrm{true}}(x))^2\Bigl[
\Bigl(\rho_t(x)\tfrac{\sigma_t^{\mathrm{ML,all}}(x)}{\sigma_t^{\mathrm{ML}}(x)}\Bigr)^2 + 1-\rho_t(x)^2
\Bigr].
\end{aligned}
\end{equation}
The posterior uncertainty under our algorithm is always weakly smaller than that under the Vanilla GP-UCB algorithm. The online stage of the algorithm uses the same optimism structure as GP-UCB but with a smaller uncertainty achieved by prediction data.

\begin{algorithm}[t]
\caption{PA-GP-UCB}\label{alg:gp-mla-ucb}
\begin{algorithmic}[1]

\STATE \textbf{Input:} domain $\mathcal{X}\subseteq[0,r]^d$, radius $\varepsilon$, replication $N$, prior kernel $K(\cdot,\cdot)$, correlation matrix $B$, horizon $T$.
\STATE Construct an $\varepsilon$-net $\{c_1,\ldots,c_M\}\subset\mathcal{X}$ of size $M=\lceil \frac{r}{2\varepsilon}\rceil ^d$.
\FOR{$i=1,\ldots,M$}
  \FOR{$j=1,\ldots,N$}
    \STATE 
    Observe $y_{ij}^{\mathrm{ML}} \leftarrow y^{\mathrm{ML}} (c_i)$
    
    \STATE $\mathcal{GP}_{\mathrm{all}} \leftarrow \textsc{Update}\!\left(\mathcal{GP}_{\mathrm{all}},\,(c_i,y^{\mathrm{ML}}_{ij})\right)$
  \ENDFOR
\ENDFOR
\FOR{$t=1,2,\ldots,T$} 
  \STATE $x_t \leftarrow \arg\max_{x\in\mathcal{X}} \varphi_t(x)$
  \STATE Observe $y_{t} \leftarrow y(x_t)$
  \STATE Observe  $y_{t}^{\mathrm{ML}} \leftarrow y^{\mathrm{ML}} (x_t)$
  \STATE $\mathcal{GP}_{\mathrm{all}} \leftarrow \textsc{Update}\!\left(\mathcal{GP}_{\mathrm{all}},\,(x_t,y_t,y^{\mathrm{ML}}_t)\right)$
  \STATE $\mathcal{GP} \leftarrow \textsc{Update}\!\left(\mathcal{GP},\,(x_t,y_t,y^{\mathrm{ML}}_t)\right)$
\ENDFOR

\end{algorithmic}
\end{algorithm}



\begin{figure*}[t]
  \centering

      \vspace{0.8em}
\begin{overpic}[width=0.95\linewidth]{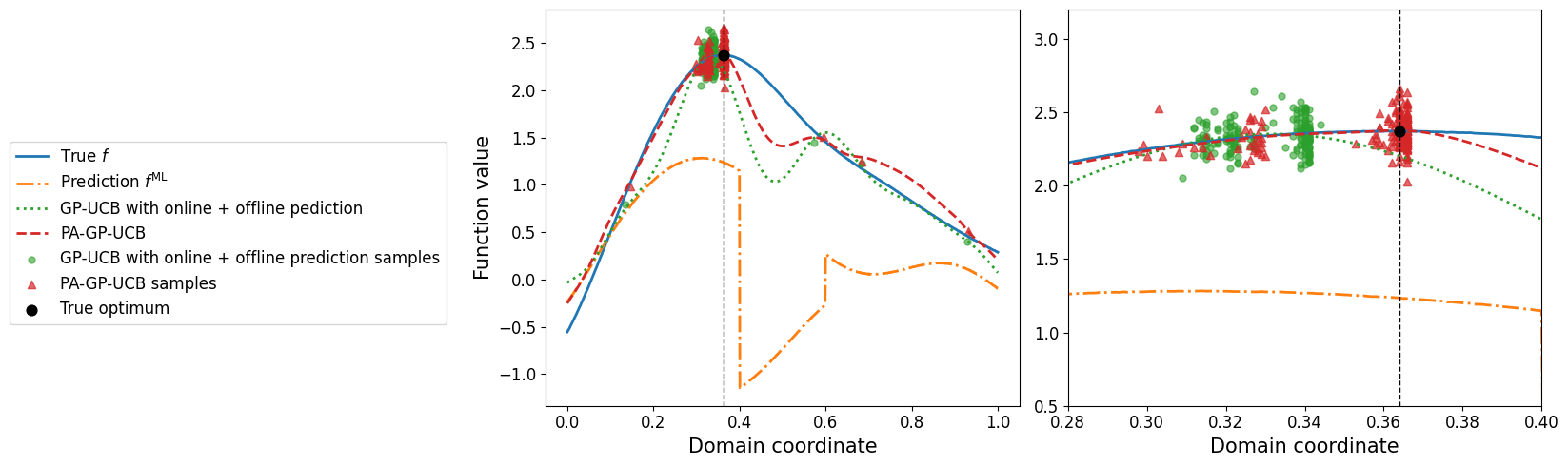}
    \put(49,30){\small\textit{(a)}}
    \put(82,30){\small\textit{(b)}}
\end{overpic}

\caption{
(a) Posterior means at $T=200$ for PA-GP-UCB and na\"{\i}ve prediction-augmented GP-UCB compared
with the ground-truth function $f$ and the locally anti-correlated prediction $f^{\mathrm{ML}}$.
(b) Zoomed-in view around the true optimum, highlighting that PA-GP-UCB
concentrates samples near the optimum while na\"{\i}ve prediction-augmented GP-UCB remains biased
toward the prediction-preferred region.
}

\label{post}
\end{figure*}


\section{Theoretical Guarantees}\label{sec:regret}
This section presents the cumulative regret analysis of PA-GP-UCB. Our bound matches the standard GP-UCB rate $\Tilde{\mathcal O}(\sqrt{dT\gamma_T})$, where  $\gamma_T$ is the information gain of using $T$ number of only
ground-truth oracle queries (without prediction oracle queries), same as in \cite{Srinivas2010}. See definition of $\gamma_T$ in (\ref{eq:info-gain-definition}). Moreover, our algorithm obtains a strictly smaller leading constant whenever the prediction is informative $(\rho\neq 0)$ and offline prediction data reduces uncertainty uniformly (isolated in the constant $R$). 


The multiplicative improvement over Vanilla GP-UCB is controlled by the factor $\sqrt{1-(1-R)\rho^2}$. It equals $1$ (no gain) when $\rho=0$ or $R=1$, and approaches $0$ (maximal gain) as $\rho \rightarrow 1$, and $ R\rightarrow 0$.

\begin{theorem}[Cumulative Regret Bound] \label{thm1}
If Assumption \ref{assumption:lip} holds, and suppose the offline data satisfy the uniform ratio condition 
\[(\sigma_t^{\mathrm{ML},\text{all}}(x)/\sigma_t^{\mathrm{ML}}(x))^2\leq R\] for some $R\in(0,1]$. 

Let $
\beta_t = 2\log(2\pi^2t^2/(3\delta))+4d\log(dtbr\sqrt{\log(4da/\delta)})$, $C_1 = 8/\log(1+({(\eta^2 + \rho^2\eta_{\mathrm{ML}}^2)}/{(1-\rho^2)})^{-1})$. Then with probability $\geq1-\delta$,  PA-GP-UCB satisfies for all $T\geq 1$,
\[
R_T\leq \sqrt{C_1\,\beta_T \,T\,\left[1-(1-R)\rho^2 \right] \gamma_T} +\frac{\pi^2}{6}.
\]
\end{theorem}

\begin{proof}
    Theorem \ref{thm1} is proved in Appendix
\ref{app:proofs}.
\end{proof}

When applying the Vanilla UCB-GP algorithm described in \cite{Srinivas2010}, the cumulative regret under the same assumptions is guaranteed to have the following for all $T\geq 1$,
\[
R_T^{\text{Vanilla}} \leq \sqrt{C_2\, T\, \beta_T\, \gamma_T} + \frac{\pi^2}{6}, \quad  C_2=\frac{8}{\log(1+\eta^{-2})}.
\]

\begin{corollary}[Strictly Better Performance]
\label{cor:strictly-better} Let $R^* \triangleq  \min\left\{\frac{\frac{C_2}{C_1} -(1-\rho^2)}{\rho^2},1\right\}$. 
    When $R< R^*$ and $\rho \neq 0$, 
then running Vanilla GP-UCB algorithm gives a cumulative regret strictly larger than that after running PA-GP-UCB, i.e. for all $T\geq 1$,
\[
R_T^{\mathrm{Vanilla}}> R_T.
\]

An asymptotic and with-high-probability conservative sufficient condition is 
    \begin{equation}
    \label{eps-N-req}
        \varepsilon = \tilde{\mathcal{O}}\left( \sqrt{\frac{R^*}{T}}\frac{\eta_{\mathrm{ML}}}{d}\right), \text{ and }N = {\Omega}\left(\frac{T}{R^*}\right).
    \end{equation}
\end{corollary}

In practice, one chooses an offline budget/design (e.g., $\varepsilon$-net size and replication $N$, 
or alternative non-uniform designs) that is computationally feasible, and can empirically estimate the 
induced $\hat{R}$ on a holdout set by comparing $\sigma_t^{\mathrm{ML},\text{all}}(x)$ and $\sigma_t^{\mathrm{ML}}(x)$ for $x$ in the holdout set. Corollary \ref{cor:strictly-better} can then be read as: whenever $\hat{R}<R^*$, PA-GP-UCB provably improves the regret constant. Our experiments in later sections show substantial regret reductions even with much larger $\varepsilon$ and 
smaller $N$ than~(\ref{eps-N-req}), consistent with the fact that~(\ref{eps-N-req}) is extremely conservative and that 
performance depends on the variance reduction achieved at the points actually explored by the algorithm, rather than worst-case uniform coverage.



\section{Numerical Analysis}
\label{sec:numerical}

We first validate PA-GP-UCB on synthetic data to examine its efficiency and bias
correction under controlled conditions.
We generate a ground-truth function $f$ by sampling from a one-dimensional Gaussian
process with an RBF kernel over a  domain $D \subseteq [0,1]$.
To construct a correlated but biased prediction, we draw an independent GP sample $g$
from the same kernel and define
$
f^{\mathrm{ML}} = \rho\, f + \sqrt{1-\rho^{2}}\, g,
$
which we use as the prediction with controllable correlation $\rho$ to the
ground truth.
At each iteration, the algorithm queries both the human oracle and the prediction oracle
with different noise variances and updates a joint two-task GP posterior.

\begin{figure*}[t]
  \centering
     \vspace{-0.9em}

  \begin{subfigure}{0.32\linewidth}
    \centering
    \caption{}
    
    \includegraphics[width=\linewidth]{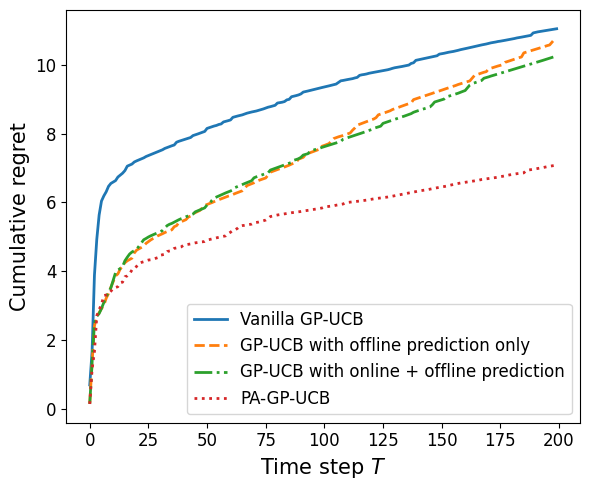}
  \end{subfigure}
  \begin{subfigure}{0.32\linewidth}
    \centering
    \caption{}
    \includegraphics[width=\linewidth]{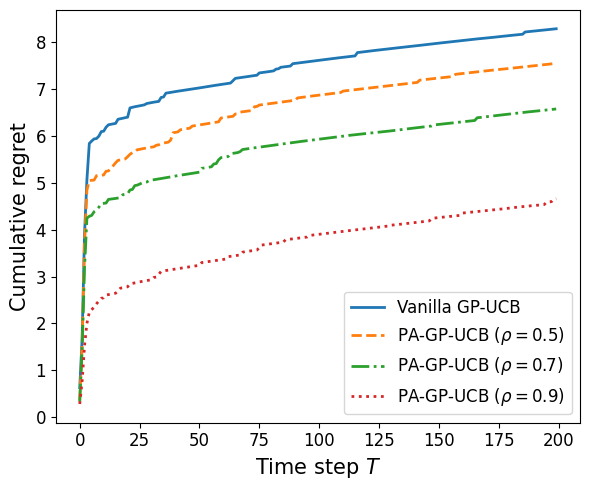}
  \end{subfigure}
  \begin{subfigure}{0.32\linewidth}
    \centering
 
    \caption{}
    \includegraphics[width=\linewidth]{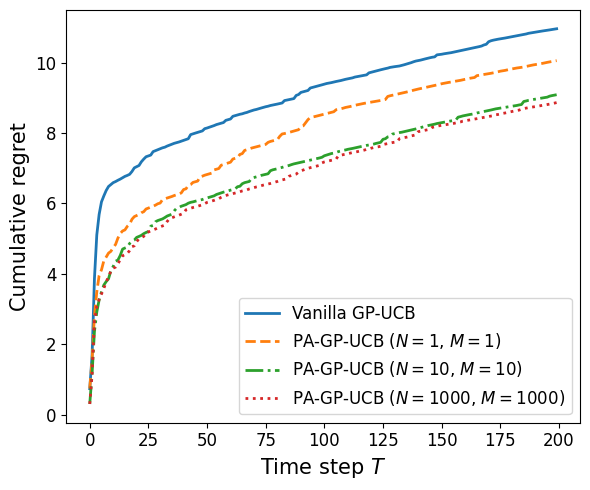}
  \end{subfigure}

\caption{
Cumulative regret of Vanilla GP-UCB, na\"{\i}ve prediction-augmented GP-UCB baselines, and PA-GP-UCB,
averaged over $50$ runs with horizon $T=200$.
(a) Comparison of Vanilla GP-UCB, GP-UCB with offline predictions only,
GP-UCB with both online and offline predictions, and PA-GP-UCB under
$\rho = 0.8$, $\eta^2 = \eta_{\mathrm{ML}}^2 = 0.01$, and $M = N = 1000$.
(b) Effect of correlation $\rho \in \{0.5, 0.7, 0.9\}$ with
$\eta^2 = \eta_{\mathrm{ML}}^2 = 0.001$ and $M = N = 1000$.
(c) Effect of varying the offline prediction data sizes $M$ and $N$
with $\eta^2 = \eta_{\mathrm{ML}}^2 = 0.01$.
}

\label{fig:exp12}
\end{figure*}
\subsection{Comparison to Na\"{\i}ve Prediction-Augmented Baselines}
In addition to Vanilla GP-UCB, we consider two na\"{\i}ve prediction-augmented baselines:
(i) GP-UCB with offline prediction data, which incorporates only pre-collected
prediction observations into the GP prior, and
(ii) GP-UCB with offline and online prediction data, which augments the GP with
both offline predictions and prediction observations collected online but does not
correct for prediction bias.
These baselines allow us to separately evaluate the contribution of variance reduction from prediction data and the additional benefit of explicit bias correction.

To evaluate robustness under structured prediction misspecification, we construct a
synthetic setting in which the predictions are globally correlated with the true objective
but locally misleading.
Specifically, we flip the sign of the prediction within a sub-interval $[0.4,0.6]$, i.e.,
$f^{\mathrm{ML}}(x)\leftarrow -f^{\mathrm{ML}}(x)$ in this region.
This preserves global correlation while inducing a localized regime in which the
prediction is confidently anti-correlated with the true objective. Figure~\ref{post} and Figure~\ref{fig:exp12}(a) show that GP-UCB augmented
with offline, or offline and online, prediction data initially achieves relatively low cumulative
regret by exploiting prediction guidance, but its posterior remains biased toward the
prediction-preferred region and fails to recover the true optimum.
In contrast, PA-GP-UCB uses online prediction feedback to estimate local correlation and
explicitly correct the prediction-induced shift in the human posterior, enabling it to
reallocate exploration toward the true optimum and achieve both lower asymptotic regret
and accurate localization of the optimum.

\subsection{Ablation on Prediction Correlation and Offline Coverage}

Figure~\ref{fig:exp12}(b) shows that PA-GP-UCB consistently outperforms Vanilla GP-UCB
across all correlation levels.
Notably, even with moderate correlation ($\rho = 0.5$), PA-GP-UCB achieves substantial
improvements, indicating that strong alignment between predictions and the ground-truth
objective is not required in practice.
Higher values of $\rho$ further amplify this advantage by enabling more effective
information transfer from the offline predictor.

Figure~\ref{fig:exp12}(c) varies the number of offline $\varepsilon$-net size $M$ and repetitions $N$.
PA-GP-UCB outperforms Vanilla GP-UCB across all settings, including the minimal regime
$M = N = 1$, suggesting that even a very small amount of offline prediction data is
sufficient to yield meaningful performance gains in practice, without requiring the
strong conditions assumed by the theoretical guarantees.
Increasing $M$ and $N$ further improves performance by reducing posterior uncertainty
in the prediction oracle.

Further posterior visualizations illustrating the temporal evolution of PA-GP-UCB
are provided in Appendix~\ref{addpost}.
Additional ablation studies examining the effects of observation noise and prediction
noise are reported in Appendix~\ref{moreablations}.

\section{Application to Hypothesis Generation}
\label{sec:experiments}

We now apply PA-GP-UCB to the problem of hypothesis generation, where the goal is to discover high-quality hypotheses with minimal human feedback or evaluation of the ground truth. Each hypothesis corresponds to a point $x$ in a continuous embedding space $\mathcal{X} \subset \mathbb{R}^d$ (e.g., the latent representation of a scientific idea, design concept, or policy proposal). A human oracle provides expensive, noisy ground-truth feedback $f(x)$, while machine learning predictions (e.g., LLMs) provide cheap, biased predictions $f^{\mathrm{ML}}(x)$. The key challenge is to efficiently explore $\mathcal{X}$ and converge to the hypothesis that maximizes the ground truth function, using the prediction to accelerate exploration while correcting for its bias \cite{Zhou_2024}. PA-GP-UCB is particularly well-suited for this task because it naturally generates novel candidate hypotheses by maximizing the UCB of the bias-corrected posterior, enabling extrapolation beyond the initial set of human-proposed ideas.  By combining human anchoring with
prediction-accelerated exploration, PA-GP-UCB provides an efficient framework
for human--AI co-discovery with provably faster convergence than standard
GP--UCB.


We evaluate PA-GP-UCB in two complementary environments, both grounded in
the real human behavioral dataset of \cite{milkman2021megastudies}. The dataset
contains large-scale field measurements of how different experimental messaging
conditions influence user visitation behavior. We compute the mean human visitation
reward for each experimental condition:
$
    f(a_i)
    =
    \mathbb{E}[\text{visits} \mid \text{exp\_condition} = a_i],
    i = 1,\dots, n.
$ This yields $n= 54$ distinct human-defined ``arms'' $\{a_i\}$ with
corresponding ground-truth human outcomes. These same arms form the basis of
our two evaluation settings.  
First, in a finite-arm setting, the algorithm operates directly
on the 54 discrete human-measured hypotheses, yielding a bandit problem with no
underlying smoothness structure.  
Second, we construct a human-derived synthetic environment, a
structure-preserving continuous benchmark obtained by embedding the same 54 arms
into a two-dimensional semantic manifold and smoothing their human reward values
into a continuous function.  
\begin{figure*}[t]
  \centering
     \vspace{-0.9em}
  \begin{subfigure}{0.32\linewidth}
  \centering
  \caption{}
  \includegraphics[width=\linewidth]{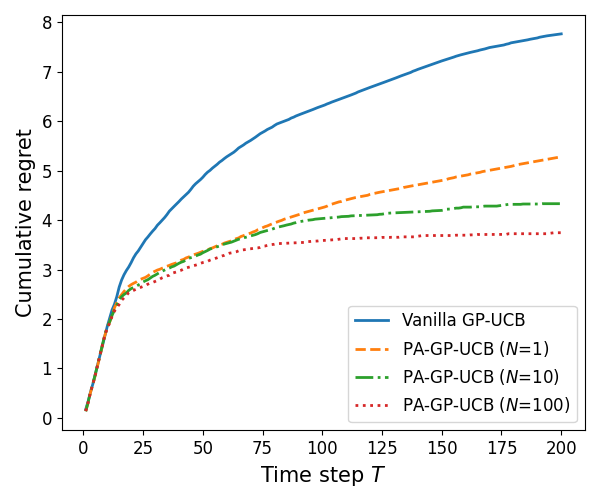}
\end{subfigure}
  \begin{subfigure}{0.32\linewidth}
  \centering
  \caption{}
  \includegraphics[width=\linewidth]{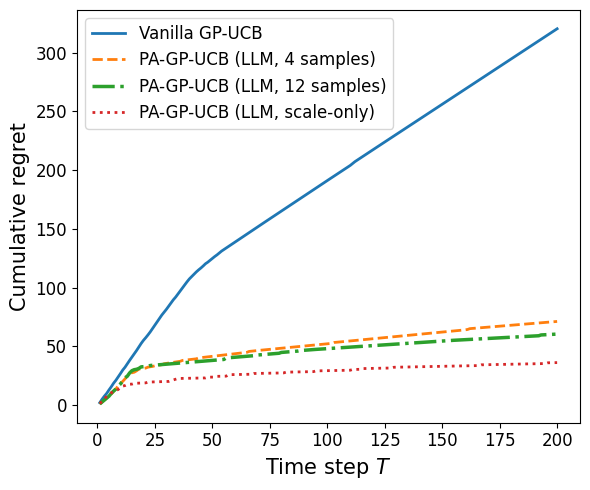}
\end{subfigure}

   \caption{
Cumulative regret over horizon $T = 200$, averaged over $50$ independent runs.
(a) Finite-arm setting: comparison between Vanilla GP-UCB and PA-GP-UCB for different
numbers of offline prediction duplicates $N \in \{1, 10, 100\}$, with
$\eta^2 = \eta_{\mathrm{ML}}^2 = 0.001$ and $\hat \rho = 0.66$.
(b) Continuous setting: comparison between Vanilla GP-UCB and PA-GP-UCB using an LLM-based
prediction, where the prediction is conditioned on different numbers of ground-truth
examples via in-context prompting, with
$\eta^2 = \eta_{\mathrm{ML}}^2 = 0.001$ and $\hat{\rho}= 0.8$.
}

  \label{exp21}
\end{figure*}

\subsection{Finite-Arm Hypothesis Evaluation}
We treat the set of hypotheses for which we have real experimental outcomes as a 
finite action set 
$
\mathcal{A} = \{x_1, \dots, x_n\},
$ where each $x_i$ is an embedding of an intervention or hypothesis description. We define
$
f(x_i) = f(a_i), i = 1,\dots, n
$
treating the collected experimental data as the true underlying human feedback function. We construct $f^{\mathrm{ML}}$ using a lightweight regression model trained on the
embedding space, where only half of the arms are visible during training and,
for those visible arms, only half of their observations are used. This mimics the real-world scenarios where $f^{\mathrm{ML}}$ takes less information and is correlated to $f$ with bias. The regret is computed with respect to the best human-evaluated hypothesis 
$x^* = \arg\max_{x_i \in \mathcal{A}} f(x_i)
$ and $
R_T = \sum_{t=1}^T \big(f(x^*) - f(x_t)\big).
$






\subsubsection{Results}
Across the $54$ experimental conditions, the prediction model achieves a moderate
empirical correlation of $\hat \rho = 0.66$ with the ground-truth human reward.
Figure~\ref{exp21}(a) shows cumulative regret over $T = 200$ for
PA-GP-UCB with varying numbers of offline prediction duplicates ($1$, $10$, and
$100$).
Across all settings, PA-GP-UCB exhibits substantially lower sublinear regret than
Vanilla GP--UCB.
The performance gap widens as the number of offline prediction duplicates increases,
since additional prediction observations provide stronger global structure and reduce
posterior uncertainty in regions where human data are sparse. Notably, even though $\hat \rho$ is an empirical estimate that may vary with $x$, using this estimated correlation is sufficient to significantly reduce regret, highlighting the robustness of PA-GP-UCB to heterogeneous prediction alignment.

In this discrete bandit setting, the true reward is highly non-smooth and does not
admit a meaningful GP structure, causing Vanilla GP--UCB to exhibit large posterior
uncertainty and unstable exploration. PA-GP-UCB instead leverages a noisy but
globally informative prediction model to impose latent structure over the action
space. The resulting two-task GP reduces posterior variance early, while the control-variates estimator debiases the prediction and balances information across tasks, yielding
more stable exploration and faster convergence despite the discrete, non-GP
nature of the reward.

\subsection{Structure-Preserving Continuous Hypothesis Evaluation}

To construct a realistic yet fully controllable benchmark, we build a
continuous two-dimensional environment directly from the real human
evaluation dataset. This preserves the semantic geometry of human hypotheses
while enabling dense sampling and closed-form regret computation. Each arm description $a_i$ is embedded through a TF--IDF vectorizer \cite{SALTON1988513} with
bi-grams
$
    v_i = \mathrm{TFIDF}(a_i) \in \mathbb{R}^d .
$
We reduce the TF--IDF vectors to a human-structured two-dimensional manifold
$
    x_i = \mathrm{UMAP}(v_i) \in [0,1]^2
$ via UMAP \cite{mcinnes2020umapuniformmanifoldapproximation}.
The embedding is rescaled to $[0,1]^2$ along each
dimension. UMAP clusters semantically similar interventions into coherent regions and produces a low-dimensional manifold enabling generalization beyond the 54 observed arms. This construction relies only on corpus-local statistics, avoiding pretrained semantic embeddings (e.g., word2vec or sentence transformers) that could introduce additional inductive bias.

We fit a Gaussian process
$
    f(x) \sim \mathcal{GP}\!\left(
        0, \;
        \sigma_f^2 \exp\!\left( -\tfrac{1}{2\ell^2}\|x-x'\|^2 \right)
    \right)
$
to the $54$ human-labeled points $\{(x_i, f(a_i))\}$.  
The GP posterior mean, evaluated on a dense $100 \times 100$ grid
$
    D_{\mathrm{grid}} = \{x_1,\dots,x_M\} \subset [0,1]^2 , M=10{,}000 ,
$
serves as the continuous human reward function
$
f^{\mathrm{GP}}(x).
$ 
To construct a cheap but imperfect auxiliary model, we use a LLM (OpenAI \texttt{o4-mini} model \cite{openai_o4mini_2025})
as a predictor of human reward.
For each arm $x_i$, we define
$
    f^{\mathrm{ML}}(x_i) \;\triangleq\; \mathbb{E}_{\text{LLM}}\!\left[\text{predicted visits}\mid \text{prompt}(x_i)\right],
$
where the expectation is approximated by a single fixed-prompt LLM query (the full prompt templates are provided in Appendix~\ref{app:llm-prompts}). We consider three prediction variants based on the amount of ground-truth information
provided via in-context prompting: (i) a $K$-shot calibrated prediction with $K=4$
ground-truth examples, (ii) a stronger $K$-shot prediction with $K=12$ examples, and
(iii) a scale-only prediction that receives only global summary statistics (min, max,
mean) without labeled examples.
Each variant induces a biased but informative predictor with increasing alignment to
the human reward as $K$ grows. Across all three variants,
$
   \mathrm{corr}(f^{\mathrm{ML}}, f^{\mathrm{GP}}) \approx 0.8,
$
indicating that the prediction is informative but systematically inaccurate.


The true optimum of the ground-truth GP surface and the cumulative regret are
$
    x^\star = \arg\max_{x\in D_{\mathrm{grid}}} f^{\mathrm{GP}}(x)$ and $R_T = \sum_{t=1}^{T}
    \big( f^{\mathrm{GP}}(x^\star) - f^{\mathrm{GP}}(x_t) \big).
$ Given the best point $x_{\mathrm{best}}$ discovered by PA-GP-UCB on the continuous
grid, we map it back to a set of nearby real experimental conditions.
Specifically, we find the indices of the $k$ closest arms in the 2-D UMAP
embedding:
$
    \mathcal{N}_k(x_{\mathrm{best}}) 
    = \Bigl\{ i \in \{1,\dots,n\} : 
      x_i \text{ is among the $k$ nearest neighbors of } x_{\mathrm{best}} \Bigr\}.
$
We then take the original textual descriptions 
$\{a_i : i \in \mathcal{N}_k(x_{\mathrm{best}})\}$ and prompt an LLM to
summarize the common structure and strategy they represent (e.g., shared
targeting rules, interventions, or messaging patterns). This produces an
interpretable, human-readable explanation of the hypothesis $H^\star$ that PA-GP-UCB discovers in
the continuous environment.

  

\subsubsection{Results}
Mapping the PA-GP-UCB optimum $x_{\mathrm{best}}$ back to the discrete arms
reveals its $k=6$ nearest neighbors: \textit{Higher Incentives a}, \textit{Higher
Incentives b}, \textit{Gain-Framed Micro-Incentives}, \textit{Choice of
Gain- or Loss-Framed Micro-Incentives}, \textit{Loss-Framed
Micro-Incentives}, and \textit{Planning, Reminders \& Micro-Incentives to
Exercise}, which is consistent with the ground-truth human outcomes.

$H^\star$: 
Interventions that combine relatively large, frequent, performance-contingent monetary incentives, 
together with basic planning and reminder scaffolding, are especially effective at increasing gym visits.

Figure~\ref{exp21}(b) shows cumulative regret for Vanilla GP--UCB
and PA-GP-UCB under different LLM prediction variants. Increasing the number of
ground-truth examples used to calibrate the prediction improves performance,
indicating that better prediction alignment yields more efficient exploration.
Notably, the scale-only prediction performs best, as it provides reliable global
magnitude information without overfitting to a small labeled set. Notably, Vanilla GP--UCB exhibits nearly linear regret due to posterior misspecification arising from a mismatch between the GP kernel and the latent semantic geometry of the reward function, whereas PA-GP-UCB achieves stable, sublinear regret by regularizing posterior uncertainty with a correlated auxiliary signal through a prediction-informed two-task GP, making it substantially less sensitive to kernel misspecification.

More generally, PA-GP-UCB operates over an abstract optimization space 
$\mathcal{X}$ and does not require explicit dimensionality reduction or lossy embeddings. Hypotheses may be high-dimensional or discrete, provided an encoding defines a similarity structure and admits exact decoding. This makes the framework particularly well suited to scientific hypothesis generation, where structured hypotheses induce meaningful geometry, data are sparse, models are often misspecified, and prediction-augmented uncertainty reduction offers substantial gains over standard bandit methods.


\section{Conclusions}
We propose PA-GP-UCB, a
GP bandit algorithm that integrates a cheap but biased
prediction oracle and offline data to obtain provable gains in sample efficiency. Under standard GP-UCB assumptions, PA-GP-UCB preserves the usual
$\tilde{\mathcal O}(\sqrt{T\beta_T\gamma_T})$ regret rate while achieving a
strictly smaller leading constant, explicitly controlled by prediction quality
and offline data coverage.
 Empirically, PA-GP-UCB converges faster than Vanilla GP-UCB and na\"{\i}ve prediction-augmented GP-UCB baselines on synthetic benchmarks.
 We further
demonstrate its utility as a novel framework for hypothesis generation through
experiments grounded in real human outcomes. In this setting, PA-GP-UCB provides a principled, end-to-end integration of
hypothesis generation and evaluation in embedded continuous hypothesis spaces.

One open question is whether the sufficient conditions on $(\varepsilon, N)$ can be significantly weakened, as they are conservative and likely not sample-optimal. These conditions are most relevant when offline predictions are orders of magnitude cheaper than online evaluations and can be heavily parallelized. A key direction for future work is improving offline data efficiency, for example via adaptive or non-uniform designs or guarantees based on trajectory-dependent variance reduction rather than a uniform worst-case $R$. Additional directions include theoretical guarantees of robustness to kernel and prediction misspecification, cost-aware offline sampling that optimizes the offline--online tradeoff, and applications to large-scale scientific hypothesis generation problems with scarce and expensive human evaluation.

\section*{Impact Statement}
This work advances the field of machine learning with a focus on Bayesian optimization,
prediction-augmented algorithms, and Gaussian process bandit optimization.
By improving online sample efficiency and reducing reliance on expensive evaluations,
the proposed methods have the potential to lower time, computational, and resource costs
in real-world applications such as scientific experimentation, engineering design, and
human-in-the-loop decision-making.
More efficient optimization procedures can enable faster scientific discovery,
reduce experimental waste, and make resource-intensive workflows more accessible.
While the techniques developed in this paper are methodological in nature, their
broader impact lies in supporting more efficient and responsible use of costly
human, computational, and physical resources across a range of domains.
The methodologies developed in this paper do not raise
specific ethical concerns beyond those commonly associated with the responsible use of
machine learning systems.

\bibliography{example_paper}
\bibliographystyle{icml2026}

\newpage
\appendix
\onecolumn

\section{Related Work}
\label{related}

The classical Gaussian Process Upper Confidence Bound (GP-UCB) algorithm is proposed in \cite{Srinivas2010}, which gives a $\mathcal{O}(\sqrt{dT\gamma_T})$ regret bound, with improvements and variations studied in \cite{Chowdhury_2017}. GP-UCB is a Upper Confidence Bound (UCB)-style  Bayesian algorithm \cite{auer2002nonstochastic, auer2002using}. Our work builds upon the Vanilla GP-UCB algorithm, with the addition of oracle access to a biased but correlated predicted reward model. This augmentation of the algorithm reduces the cumulative regret of GP-UCB by a constant $\in (0,1)$. Other lines of work in Gaussian Process (GP) bandit optimization consider different oracle access models or GP-model assumptions \cite{iwazaki2023failure, bogunovic2016time}, which are orthogonal to ours as we consider a setting augmented by predicted feedback.

Meanwhile, a separate line of work focuses on enhancing statistical inference and decision-making with machine-learned predictions. The Prediction-Powered Inference (PPI) framework \cite{angelopoulos2023prediction} uses an offline-trained model to generate a mean predictor for bias correction. Other papers have also explored the idea of incorporating a prediction model in various inference tasks \cite{robins1994estimation, pepe1992inference, scharfstein1999adjusting, athey2025surrogate}, where some literature refers to the prediction model as the surrogate model. This idea was extended to bandits by \cite{ji2025multiarmedbanditsmachinelearninggenerated}, whose Machine Learning-Assisted Upper Confidence Bound (MLA-UCB) algorithm uses an offline model to accelerate exploration in discrete multi-armed bandit problems.

While effective for inference, PPI does not address sequential decision-making or
exploration in optimization problems. While MLA-UCB extends this idea to a
sequential setting by incorporating machine-learned predictions into multi-armed
bandits, MLA-UCB treats predictions as arm-level advice in a discrete action space and
does not exploit structural or geometric relationships between arms, nor does it extend to continuous domains. PA-GP-UCB takes a different approach by integrating predictions directly into a joint
Gaussian process model over high- and low-fidelity signals.
The resulting covariance structure governs both exploration and interpolation in a
continuous domain.
Within this framework, PA-GP-UCB employs a PPI-style control-variates estimator inside
the bandit algorithm, enabling principled variance reduction while supporting
sample-efficient exploration.


Our work connects to recent applications of UCB-style algorithms for hypothesis generation with LLMs. Specifically, \cite{Zhou_2024} proposes a UCB-inspired method to enhance LLM-based hypothesis generation, supported by empirical experiments. This builds on growing interest in using LLMs for scientific hypothesis generation \cite{griffiths2019constrainedbayesianoptimizationautomatic} and as prediction or surrogate models in optimization \cite{yang2024largelanguagemodelsoptimizers}. Our paper formalizes the hypothesis generation process in a continuous embedding space into a mathematical framework with provable guarantees, extending the empirical findings of \cite{Zhou_2024} and providing theoretical grounding for the use of bias-corrected predictions in Bayesian optimization. 

Finally, our work is also closely related but orthogonal to literature on multi-fidelity bandit algorithms \cite{kandasamy2016gaussian, doi:10.1137/120884122,Forrester2007Multifidelity,yan2012active,swersky2013multi, foumani2023multi}. The multi-fidelity literature considers settings with multiple oracle access channels
of differing costs, where they select which single fidelity to query at each
round and performance is typically measured as cumulative regret normalized by total
query cost \cite{Forrester2007Multifidelity}. Our work addresses a fundamentally different regime in which the prediction cost
is negligible relative to the ground-truth oracle, justifying the simultaneous
observation of both oracles at each online query. While multi-fidelity methods focus on optimizing
cost allocation across fidelities, they are not designed to improve regret relative to
Vanilla GP optimization when performance is measured solely by the number of
high-fidelity queries. In contrast, we leverage the low-fidelity oracle as side
information to tighten confidence bounds and achieve strictly improved cumulative
regret under a single-fidelity evaluation budget. Finally, unlike classical
multi-fidelity formulations that assume uniformly bounded bias between low- and
high-fidelity evaluations, our approach imposes no such bias constraint and instead
requires only nonzero correlation between the two oracle outputs.

\section{Proof of Theorem \ref{thm1}} 
\label{app:proofs}

\subsection{Decomposing the Estimation Error}
\begin{lemma}   [Error decomposition]\label{lem:error_decomp}
  Conditioned on previous online queries $\mathcal{D}_{t-1}$,
    
    we have \[\mu^{\text{PA}}_t(x)-f(x)\triangleq E_1(x)+E_2(x),\] 
    where 
    \[
    E_1(x) = \frac{\rho_t(x)\sigma_t^{\mathrm{true}}(x)}{\sigma_t^{\mathrm{ML}}(x)} \big(\mu_t^{\mathrm{ML},\text{all}}(x)-f^{\mathrm{ML}}(x) \big), \quad E_2(x)\sim \mathcal{N}(0,(\sigma_t^{\mathrm{true}}(x))^2(1-\rho_t(x)^2)),\]
     and $E_1 \perp E_2 \mid \mathcal{D}_{t-1}$.
\end{lemma}
\begin{proof}
    The following derivations are all conditioned on previous online queries $\mathcal{D}_{t-1}$, which  implies the posterior estimates are conditionally deterministic at time $t$. The posterior distribution of the joint GP at point $x$ is a bivariate Gaussian distribution, 
    \[
\begin{bmatrix}
f\\[2pt]
f^{\mathrm{ML}}
\end{bmatrix}
\sim 
\mathcal{N}\!\left(
\begin{bmatrix}
\mu_t^{\mathrm{true}}(x)\\[2pt]
\mu_t^{\mathrm{ML}}(x)
\end{bmatrix},\;
\begin{bmatrix}
(\sigma_t^{\mathrm{true}}(x))^2 & \rho_t(x)\sigma_t^{\mathrm{ML}}(x)\sigma_t^{\mathrm{true}}(x)\\
\rho_t(x)\sigma_t^{\mathrm{ML}}(x)\sigma_t^{\mathrm{true}}(x) & (\sigma_t^{\mathrm{ML}}(x))^2
\end{bmatrix} \right).
\]
    
    Applying the conditional distribution formula of bivariate Gaussian variables, we get that, 
    \[
    f(x)\mid f^{\mathrm{ML}}(x)=y \sim \mathcal{N}\left(\mu_t^{\mathrm{true}}(x)+\rho_t(x)\frac{\sigma_t^{\mathrm{true}}(x)}{\sigma_t^{\mathrm{ML}}(x)}(y-\mu_t^{\mathrm{ML}}(x)),  (\sigma_t^{\mathrm{true}}(x))^2(1-\rho_t(x)^2) \right),
    \]
which we can rewrite as, \[
f(x)=\mu_t^{\mathrm{true}}(x)+\rho_t(x)\frac{\sigma_t^{\mathrm{true}}(x)}{\sigma_t^{\mathrm{ML}}(x)}(f^{\mathrm{ML}}(x)-\mu_t^{\mathrm{ML}}(x)) + E_2(x),
\]
where $E_2(x) \sim \mathcal{N}(0,(\sigma_t^{\mathrm{true}}(x))^2(1-\rho_t(x)^2))$, and $E_2(x)$ is independent to the other terms.

Therefore, 
\begin{align*}
    \mu_t^{\text{PA}}(x)-f(x) = \rho_t(x)\frac{\sigma_t^{\mathrm{true}}(x)}{\sigma_t^{\mathrm{ML}}(x)}(\mu_t^{\mathrm{ML},\text{all}}(x)-f^{\mathrm{ML}}(x)) + E_2(x),
\end{align*}
with $E_2(x)$ defined above, and is independent of the other terms conditioned on previous online  queries.
\end{proof}

\begin{lemma}
\label{prop:error_dist}
    (Distribution of error) Conditioned on previous online data $\mathcal{D}_{t-1}$, we have
    \[
    \mu_t^{\text{PA}}(x)-f(x)\sim \mathcal{N}\left(0, (\sigma_t^{\text{PA}}(x))^2 \right),
    \]
    where 
    \[
    (\sigma_t^{\text{PA}}(x))^2 =   (\sigma_t^{\mathrm{true}}(x))^2\left[ (\rho_t(x))^2\frac{(\sigma_t^{\mathrm{ML},\text{all}}(x))^2}{(\sigma_t^{\mathrm{ML}}(x))^2} + \left(1-(\rho_t(x))^2\right)\right]\leq(\sigma_t^{\mathrm{true}}(x))^2 .
    \]
\end{lemma}
\begin{proof}
    We first analyze the distribution of $E_1$. From Lemma \ref{lem:error_decomp} we know that $ E_1 = \frac{\rho_t(x)\sigma_t^{\mathrm{true}}(x)}{\sigma_t^{\mathrm{ML}}(x)} \big(\mu_t^{\mathrm{ML},\text{all}}(x)-f^{\mathrm{ML}}(x) \big)$ . 

    The following arguments are all conditional on online data $\mathcal{D}_{t-1}$.  Let
$X^{\mathrm{off}}$ denote the points queried in the offline stage 
and $y^{\mathrm{ML},\mathrm{off}} \in \mathbb{R}^{n_{\mathrm{off}}}$ denote the
offline prediction observations with
$y^{\mathrm{ML},\mathrm{off}}=f^{\mathrm{ML}}(X^{\mathrm{off}})+\boldsymbol{\varepsilon}$,
$\boldsymbol{\varepsilon}\sim\mathcal{N}(0,\eta_{\mathrm{ML}}^2 I)$ independent.
Define
\[
\mu_t^{\mathrm{ML},\mathrm{off}\mid\mathrm{on}}
\triangleq\mathbb{E}[y^{\mathrm{ML},\mathrm{off}}\mid \mathcal{D}_{t-1}],
\qquad
r\triangleq y^{\mathrm{ML},\mathrm{off}}-\mu_t^{\mathrm{ML},\mathrm{off}\mid\mathrm{on}}.
\]
Let
\[
K(x)\triangleq\mathrm{Cov}\!\big(f^{\mathrm{ML}}(x),f^{\mathrm{ML}}(X^{\mathrm{off}})\mid \mathcal{D}_{t-1}\big),
\qquad
S\triangleq\mathrm{Cov}\!\big(f^{\mathrm{ML}}(X^{\mathrm{off}}),f^{\mathrm{ML}}(X^{\mathrm{off}})\mid \mathcal{D}_{t-1}\big)+\eta_{\mathrm{ML}}^2 I.
\]
Then $r\sim\mathcal{N}(0,S)$ and the posterior after adding
$\mathcal{D}^{\mathrm{off}}$ satisfies
\[
\mu_t^{\mathrm{ML},\mathrm{all}}(x)=\mu_t^{\mathrm{ML}}(x)+K(x)^\top S^{-1}r.
\]

    Thus, \begin{align*}
   \mu_t^{\mathrm{ML},\text{all}}- f^{\mathrm{ML}}(x) = K(x)^\intercal S^{-1}r+(\mu_t^{\mathrm{ML}}(x)- f^{\mathrm{ML}}(x) ).
\end{align*}
Now compute the following conditional on $\mathcal{D}_{t-1}$,
\begin{align*}
\Var&\left( \mu_t^{\mathrm{ML},\text{all}}(x)- f^{\mathrm{ML}}(x)
\right)    = \Var(K(x)^\intercal S^{-1}r+\mu_t^{\mathrm{ML}}(x)- f^{\mathrm{ML}}(x) )\\
& = \Var(\mu_t^{\mathrm{ML}}(x)- f^{\mathrm{ML}}(x)) - 2\Cov (\mu_t^{\mathrm{ML}}(x)- f^{\mathrm{ML}}(x) ,K(x)^\intercal S^{-1}r) + \Var(K(x)^\intercal S^{-1}r)\\
& = (\sigma_t^{\mathrm{ML}}(x))^2 - 2K(x)S^{-1}K(x)^\intercal + K(x)S^{-1}SS^{-1}K(x)^\intercal \\
& = (\sigma_t^{\mathrm{ML}}(x))^2 - K(x)S^{-1}K(x)^\intercal.
\end{align*}
The resulting expression, by the Bayesian updating formula, is exactly $\Var(f^{\mathrm{ML}}(x)\mid \mathcal{D}_{t-1}\cup \mathcal{D}^{\text{off}})\triangleq(\sigma_t^{\mathrm{ML},\text{all}}(x))^2$. Moreover, note that $\mu_t^{\mathrm{ML},\text{all}}(x)- f^{\mathrm{ML}}(x)$ has mean $0$ because both $r$ and $\mu_t^{\mathrm{ML}}(x)- f^{\mathrm{ML}}(x)$ have mean $0$.

Thus, we get that 
\[
E_1(x) \sim \mathcal{N}\left(0, \left[\rho_t(x)\sigma_t^{\mathrm{true}}(x)\frac{\sigma_t^{\mathrm{ML},\text{all}}(x)}{\sigma_t^{\mathrm{ML}}(x)} \right]^2\right).
\]

Finally, since $E_1(x)$ and $E_2(x)$ are independent, the resulting error distribution in the Lemma follows.


\end{proof}

\subsection{Sufficient Conditions on Offline Data}
\begin{lemma} (Requirement on offline data)\label{lem:offline_req} Let $L=b\sqrt{\log(da/\delta)}$, and if 
\[
\varepsilon \leq \sqrt{\frac{(\sigma_{\min}^{\mathrm{ML}}(T))^2\,R}{2L^2d^2 }}  \text{  and  }  N\geq \frac{2\eta_{\mathrm{ML}}^2}{(\sigma_{\min}^{\mathrm{ML}}(T))^2R},
\]
then with probability at least $1-\delta$, 
\[
\left(\frac{\sigma_t^{\mathrm{ML},\text{all}}(x)}{\sigma_t^{\mathrm{ML}}(x)}\right)^2\leq R  \quad \forall x\in \mathcal X,t\in [T].
 \]

 Where 
 \[(\sigma_{\min}^{\mathrm{ML}}(T))^2\triangleq \frac{\frac{1}{1-\rho^2}+\frac{T}{\eta^2}}{\frac{T^2}{\eta_{\mathrm{ML}}^2\eta^2} +\left(\frac{1}{\eta_{\mathrm{ML}}^2}+\frac{1}{\eta^2}\right)\frac{T}{K_{\min}(1-\rho^2)}.
 +\frac{1}{K_{\min}^2}(1-\rho^2) }, \text{ and } K_{\min}\triangleq {\min}_x K(x,x).\]
\end{lemma}
\begin{proof}
The online stage of the algorithm collects $T$ samples of $\{y^{\mathrm{ML}}(x_t),y(x_t)\}_{t\in [T]}$ after running for $T$ rounds. $\sigma_t^{\mathrm{ML}}(x)$ is monotonically non-increasing in $t$ for all $x\in \mathcal{X}$. Note that at any fixed point $x$, the minimum posterior variance after $T$ rounds is at least as big as when all $T$ queries are made at this point. In other words:
\begin{align*}
    \min_x\;(\sigma_T^{\mathrm{ML}}(x)) ^2 &\geq \min _x    \underbrace{\Var\left(f^{\mathrm{ML}}(x)\mid T \text{ samples of } \left[y(x),\,\, y^{\mathrm{ML}}(x)\right]\right)}_{\triangleq (\sigma^{\mathrm{ML}}(T,x))^2}.
\end{align*}
From the posterior formula, we get that \[(\sigma^{\mathrm{ML}}(T,x))^2= 
\left[\frac{1}{K(x,x)}B^{-1}+ \text{diag}\left(\frac{\eta^2}{T},\frac{\eta_{\mathrm{ML}}^2}{T}\right)^{-1}\right]^{-1}_{2,2}.
\]
Compute the individual inverses,
\[
B^{-1}=\frac{1}{1-\rho^2}\begin{bmatrix}
1 & -\rho\\
-\rho & 1
\end{bmatrix},\quad \text{diag}\left(\frac{\eta^2}{T},\frac{\eta_{\mathrm{ML}}^2}{T}\right)^{-1}=\begin{bmatrix}
\frac{T}{\eta^2} & 0\\
0 & \frac{T}{\eta_{\mathrm{ML}}^2}
\end{bmatrix}.
\]
Substituting them back, we get that 
\begin{align*}
   \min_x (\sigma_T^{\mathrm{ML}}(x))^2 &\geq\min_x(\sigma^{\mathrm{ML}}(T,x))^2\\
    & = \min_x \frac{\frac{1}{K(x,x)}\frac{1}{1-\rho^2}+\frac{T}{\eta^2}}{\left( \frac{T}{\eta_{\mathrm{ML}}^2}+\frac{1}{K(x,x)}\frac{1}{1-\rho^2}\right)\left(\frac{T}{\eta^2}+\frac{1}{K(x,x)}\frac{1}{1-\rho^2}\right) - \left(\frac{\rho}{1-\rho^2}\right)^2}\\
    & \geq \frac{\frac{1}{1-\rho^2}+\frac{T}{\eta^2}}{\frac{T^2}{\eta_{\mathrm{ML}}^2\eta^2} +\left(\frac{1}{\eta_{\mathrm{ML}}^2}+\frac{1}{\eta^2}\ \right)\frac{T}{K_{\min}(1-\rho^2)} +\frac{1}{K_{\min}^2}(1-\rho^2) }\triangleq (\sigma_{\min}^{\mathrm{ML}}(T))^2,
\end{align*}
where we define $K_{\min}\triangleq {\min}_x K(x,x)>0$ by assumption, and the last inequality is true from the assumption of $0<K\leq 1$.

Now, we upper bound $(\sigma_t^{\mathrm{ML},\text{all}}(x))^2$. Since we cannot control what exactly is queried during the online stage of the algorithm, we use the offline stage to ensure an upper bound on it. In the offline stage, we construct an $\varepsilon$-net of size $M$ over the domain $\mathcal{X}$, and at the center of each cell (denote cell as $C_i$ and center as $c_i$) query $N$ samples. Consider the following averaging estimator for any $x\in C_i$:
\begin{align*}
    \hat{f}(x)=\frac{1}{N}\sum_{j=1}^Ny^{\mathrm{ML}}(c_i) = f^{\mathrm{ML}}(c_i)+\hat{\varepsilon}, \quad \hat{\varepsilon}\sim \mathcal{N}(0,\frac{\eta_{\mathrm{ML}}^2}{N}),
\end{align*}
where $\hat{\varepsilon}$ is independent to $f^{\mathrm{ML}}(c_i)$. We can compute the mean squared error of this estimator for all $x\in  C_i$ and for all $i\in[M]$,
\begin{align*}
   \text{MSE}(\hat{f},x)&=\E[(\hat{f}(x)-f^{\mathrm{ML}}(x))^2]=({f^{\mathrm{ML}}}(c_i)-f^{\mathrm{ML}}(x))^2 + \frac{\eta_{\mathrm{ML}}^2}{N}\\
 & \leq L^2d^2\varepsilon^2 +\frac{\eta_{\mathrm{ML}}^2}{N} = b^2 \log(da/\delta)d^2 \varepsilon^2 + \frac{\eta_{\mathrm{ML}}^2}{N}.
\end{align*}
The second inequality is conditional on $L$-Lipschitzness of $f^{\mathrm{ML}}$ and $f$ as stated in Assumption \ref{assumption:lip}, which happens with probability at least $1-\delta$ when $L=b\sqrt{\log(da/\delta)}$.

It is well-known that Gaussian posterior updating of a function sampled from a zero mean GP (Kriging) is a Best Linear Unbiased Predictor (BLUP), i.e. with minimum mean squared error \cite{santner2003design}. Since we maintain the GPs during the algorithm through kriging, we get that
\begin{align*}
   (\sigma_t^{\mathrm{ML},\text{all}}(x))^2 \leq  (\sigma^{\text{off}}(x))^2 \leq   \text{MSE}(\hat{f},x),\quad \forall x\in\mathcal{X},\,t\in[T].
\end{align*}
The first inequality is true because conditioning on more data can only weakly decrease the posterior variance at all points in the domain. 

The result in this Lemma then follows from combining the upper bound on $(\sigma_t^{\mathrm{ML},\text{all}}(x))^2$ and the lower bound on $(\sigma_t^{\mathrm{ML}}(x))^2$.

\end{proof}

\subsection{Bounding the Instantaneous Regret}
\begin{lemma} [Gaussian concentration]\label{gaussian_conc} For ${\delta}\in(0,1)$ and set $\beta_t' = 2\log(2|\mathcal{X}_t|\pi_t/\delta)$, with $\sum_{t\in[T]}\pi_t^{-1}\leq 1$ and $\pi_t>0$. Then
    \[
    |f(x)-\mu_{t}^{\text{PA}}(x)|\leq \sqrt{\beta_t'} \sigma_{t}^{\text{PA}}(x) \quad \forall x\in \mathcal{X}_t\subset \mathcal X, \forall t\geq 1,
    \] with probability at least $1-\delta/2$.
\end{lemma}
\begin{proof}
       Note that for any $z \sim \mathcal{N}(0,1)$ \cite{vershynin2009high}, 
    \[
    \mathbb{P}\{|z|>c\}\leq  \exp\left(-\frac{c^2}{2} \right) .   \]

       Fix $t\in [T]$ and $x\in \mathcal{X}_t\subset \mathcal X$. From Lemma \ref{prop:error_dist} we know that $f-\mu_{t}^{\text{PA}}\mid \mathcal{D}_{t-1}\sim \mathcal{N}\left(0, (\sigma_t^{\text{PA}}(x))^2 \right)$. Thus, applying the Chernoff bound, and union bounding over all points in $\mathcal{X}_t$, we get that 
    \[
\mathbb{P}\left\{ \exists x\in \mathcal{X}_t :
\frac{|f(x)-\mu_t^{\mathrm{PA}}(x)|}{\sigma_t^{\mathrm{PA}}(x)} > \sqrt{\beta_t'}
\ \Bigm|\ \mathcal{D}_{t-1}\right\}
\leq |\mathcal{X}_t|\exp\!\left(-\frac{\beta_t'}{2}\right)\leq \frac{\delta}{2\pi_t}.
\]

    Further union bounding over all $t\in[T]$ gives the final result since we require that $\sum_{t\in[T]}\pi_t^{-1}\leq 1$.

\end{proof}

At each round $t\in[T]$, we construct an $\varepsilon$-net of $\mathcal X_t\subset \mathcal X$ and leverage the Lipschitzness assumption \ref{assumption:lip} to upper bound the instantaneous regret. The $\varepsilon$-net is defined solely for the purpose of analysis. This is a technique similarly used in the analysis of \cite{Srinivas2010}, and we cite a useful result from it below.  
\begin{lemma}[\citet{Srinivas2010} Lemma 5.7] \label{lem:discrete}  For $\delta\in(0,1)$, set $\beta_t=2\log(4\pi_t/\delta)+4d\log(dtbr\sqrt{\log(4da/\delta)})$, where $\sum_{t\geq1}\pi_t^{-1}=1$, $\pi_t>0$. Construct a $\varepsilon-$net $\mathcal{X}_t$ of size $|\mathcal{X}_t|=(dt^2br\sqrt{\log(4da/\delta)})^d$. Let $[x^*]_t$ denote the closest point in $\mathcal{X}_t$ to $x^*$. Then,
    \[
   \left |f(x^*)-\mu_{t}^{\text{PA}}([x^*]_t)\right| \leq \sqrt{\beta_t}
   \,\sigma_{t}^{\text{PA}}([x^*]_t)+\frac{1}{t^2}, \quad \forall t\in[T],
    \]
    with probability at least $1-\delta/2$.
\end{lemma}
\begin{proof}[Proof Sketch of \cref{lem:discrete}]
    We sketch out the proof of this  Lemma. Readers interested in the details of the proof can refer to \cite{Srinivas2010}. We first leverage Assumption \ref{assumption:lip}, which happens with probability at least $1-\delta/4$ when $L = b\sqrt{\log(4da/\delta)}$. By Lipschitzness, one can determine an appropriate $\varepsilon$-net size $|\mathcal X_t|$ for each round $t\in[T]$ to upper bound the difference between $f([x^*]_t)$ and $f(x^*)$  by $1/t^2$ with high probability. We now use $\delta/4$ from Lipschitzness and $\delta/4$ from Lemma \ref{gaussian_conc} to get that the event in this Lemma happens with probability at least $1-\delta/2$. We can also get an upper bound on $\beta_t'$ from  Lemma \ref{gaussian_conc} by substituting the value of  $|\mathcal X_t|$,
\begin{align*}
    2\log(4|\mathcal X_t|\pi_t/\delta) & = 2\log(4\pi_t/\delta) + 2\log(|\mathcal X_t|)\\
    & = 2\log(4\pi_t/\delta)+2\log\left(\left(dt^2br\sqrt{\log(4da/\delta)}\right)^d\right)\leq \beta_t.
\end{align*}
\end{proof}

\begin{lemma}[Instantaneous regret]\label{lem:instant_regret}
     Set $\beta_t = 2\log(4\pi_t/\delta)+4d\log(dtbr\sqrt{\log(4da/\delta)})$. Where $\sum_{t\in[T]}\pi_t^{-1}\leq1$ and $\pi_t>0$. For each $t\in[T]$, conditioned on $\mathcal{D}_{t-1}$, we have
    \[
    r_t \leq 2\sqrt{\beta_t}\sigma_t^{\text{PA}}(x_t) +\frac{1}{t^2},
    \]
    with probability at least $1-\delta$.
\end{lemma}
\begin{proof}
     A union bound over $\delta/2$ from Lemma \ref{lem:discrete} and $\delta/2$ Lemma \ref{gaussian_conc} gives that both events hold with probability greater than $1-\delta$. By definition of $x_t$ as the maximizer of $\varphi_t(x)$, we get that with probability at least $1-\delta$,
    \begin{align*}
        f(x_*)&\leq\mu_t^{\text{PA}}(x_*)+\sqrt{\beta_t}\sigma_t^{\text{PA}}(x_*)
        \\
        &\leq \mu_t^{\text{PA}}([x_*]_t)+\sqrt{\beta_t}\sigma_t^{\text{PA}}([x_*]_t) +\frac{1}{t^2}\\
        & \leq \mu_t^{\text{PA}}(x_t)+\sqrt{\beta_t}\sigma_t^{\text{PA}}(x_t) +\frac{1}{t^2}.
    \end{align*}
Note that we can apply Lemma \ref{gaussian_conc} because $\beta_t$ used here is bigger than $\beta_t'$ in Lemma \ref{gaussian_conc}. Thus,
\begin{align*}
    r_t &= f(x_*)-f(x_t)\\
    &\leq \mu_t^{\text{PA}}(x_t) - f(x_t)+\sqrt{\beta_t}\sigma_t^{\text{PA}}(x_t) +\frac{1}{t^2}\\
    &\leq 2\sqrt{\beta_t}\sigma_t^{\text{PA}}(x_t) +\frac{1}{t^2}.
\end{align*}

\end{proof}

\subsection{Proof of \cref{thm1}}
\begin{proof}[Proof of Theorem \ref{thm1}]

For this proof, we adopt the choice of $\beta_t$ from Lemma \ref{lem:instant_regret}. Since $\sum_{t\geq 1}\frac{1}{t^2}=\frac{\pi^2}{6}$, we  choose $\pi_t = \frac{\pi^2t^2}{6}$, which gives that $\sum_{t\in[T]}\pi_t^{-1} \leq \sum_{t\geq 1}\pi_t^{-1}=\frac{\pi^2}{6}\frac{6}{\pi^2}= 1$ as required. From Lemma \ref{lem:instant_regret}, we get that with probability at least $1-\delta$,
\begin{align*}
    R_T=\sum_{t\in[T]}r_t & = \sum_{t\in[T]}\left[2\sqrt{\beta_t}\sigma^{\text{PA}}_t(x_t) +\frac{1}{t^2}\right]\\
    &\leq 2\sqrt{ \left(\sum_{t\in T}\beta_t\right)\left( \sum_{t\in[T]}\left(\sigma^{\text{PA}}_t(x_t)\right)^2\right)} +\frac{\pi^2}{6} \leq 2\sqrt{T \beta_T \sum_{t\in[T]}\left(\sigma^{\text{PA}}_t(x_t)\right)^2} +\frac{\pi^2}{6}.
\end{align*}
The first inequality is due to the Cauchy Schwartz inequality, and the last inequality is because $\beta_t$ is non-decreasing in $t$.

The remainder of this proof focuses on upper bounding $ 4\sum_{t\in[T]}(\sigma_t^{\text{PA}}(x_t))^2$. Given the assumption in the Theorem that $(\sigma_t^{\mathrm{ML},\text{all}}(x)/\sigma_t^{\mathrm{ML}}(x))^2\leq R$, we get that $(\sigma^{\text{PA}}_t(x))^2\leq (1-(1-R)\rho_t^2(x))(\sigma_t^{\mathrm{true}}(x))^2$. Thus,
\begin{equation}\label{bound}
\sum_{t\in[T]}\left(\sigma^{\text{PA}}_t(x_t)\right)^2 \leq  \sum_{t\in[T]}
\left[1-(1-R)\rho_t^2(x_t)\right](\sigma_t^{\mathrm{true}}(x_t))^2.
\end{equation}
We perform a residual decomposition on $f$ in terms of two unconditionally independent Gaussian Process samples $f^{\mathrm{ML}}$ and $\xi$, with
\begin{equation}
\label{residual_decomp}
    f = \rho f^{\mathrm{ML}} + \sqrt{1-\rho^2} \xi \,,\quad \xi \sim \mathcal{GP}(0,K(x,x')).
\end{equation}

From Lemma \ref{lem:error_decomp}, we know that $(1-\rho_t^2(x))(\sigma_t^{\mathrm{true}}(x))^2 = \Var(f(x)\mid f^{\mathrm{ML}}(x),\mathcal{D}_{t-1})$,
combining with the residual decomposition (\ref{residual_decomp}),
\begin{align*}
   (1- \rho_t^2(x))(\sigma_t^{\mathrm{true}}(x))^2 &=  \Var(\rho f^{\mathrm{ML}}(x) + \sqrt{1-\rho^2} \xi(x) \mid f^{\mathrm{ML}}(x), \mathcal{D}_{t-1})\\
    & = (1-\rho^2) \Var(\xi(x)\mid f^{\mathrm{ML}}(x), \mathcal{D}_{t-1})\\
    & \leq (1-\rho^2) \Var(\xi(x)\mid  \mathcal{D}_{t-1}) = (1-\rho^2)(\sigma_t^{\xi}(x))^2.
\end{align*}
The last inequality is true because additionally conditioning on $f^{\mathrm{ML}}(x)$ can only decrease the posterior variance of $\xi$.

Rearranging this inequality gives
\begin{align}
\label{bound2}
    \sum_{t\in[T]}
\left[1-(1-R)\rho_t^2(x_t)\right](\sigma_t^{\mathrm{true}}(x_t))^2\leq \sum_{t\in[T]} (1-\rho^2)(1-R)(\sigma_t^{\xi}(x_t))^2 + R(\sigma_t^{\mathrm{true}}(x_t))^2.
\end{align}

Now recall the definition of information gain of a sample of a Gaussian Process with observation noise of variance $\sigma^2$ and kernel $k(\cdot,\cdot)$,
\begin{equation}
\label{eq:info-gain-definition}
    \gamma_T(k,\sigma^2)\triangleq \max_{A\subset \mathcal{X},\,|A|=T} \frac{1}{2}\log\det \left({I}+\sigma^{-2}K_A \right),
\end{equation}

where $K_A=[k(x,x')]_{x,x'\in A}$.

Note that by observing $y_t^s$ and $y_t^h$, we can also indirectly observe $z_t$ -- a noisy observation of $\xi(x_t)$, where
\[
z_t = \frac{y_t^h-\rho y_t^s}{\sqrt{1-\rho^2}} = \xi(x_t) + \frac{\varepsilon_t^h-\rho \varepsilon_t^s}{\sqrt{1-\rho^2}}.
\]
Thus, the noisy observation $z_t$ of $\xi(x_t)$ has variance $\frac{\eta^2 + \rho^2\eta_{\mathrm{ML}}^2}{1-\rho^2}$. 

A classical result from Lemma 5.4 in \cite{Srinivas2010} shows that for a sample $f$ of a Gaussian Process with kernel $k$ and observation noise of variance $\sigma^2$, the sum of posterior variances of $f$ after querying an arbitrary sequence of $T$ points $\{x_t\}_{t\in[T]}$ has the following property
\[
4\sum_{t\in[T]} (\sigma^f_t(x_t))^2 \leq C(\sigma^2) \gamma_T\left(k, \sigma^2 \right), \quad C(\sigma^2) = \frac{8}{\log(1+\sigma^{-2})}.
\]
Let $\tilde{\sigma}_t^{\mathrm{true}}(x)$ be the posterior standard deviation of $f$ under a single-output GP observing only the ground-truth oracle but not the prediction oracle. Since conditioning on extra data can only reduce the posterior variance, we get that
\[
\sigma_t^{\mathrm{true}}(x_t)\leq \tilde{\sigma}_t^{\mathrm{true}}(x_t) \quad \forall x\in \mathcal X, t\in[T].
\]
Applying this result and Lemma 5.4 in \cite{Srinivas2010} gives that 
\begin{equation}
\label{sum_of_sigma_f}
   4\sum_{t\in[T]}(\sigma_t^{\mathrm{true}}(x_t))^2 \leq 4\sum_{t\in[T]}(\tilde{\sigma}_t^{\mathrm{true}}(x_t))^2 \leq C(\eta^2) {\gamma_T\left(k,\eta^2\right)}\overset{(a)}{\leq}  C(\frac{\eta^2 + \rho^2\eta_{\mathrm{ML}}^2}{1-\rho^2})\underbrace{\gamma_T\left(k,\eta^2\right)}_{\triangleq\gamma_T}, 
\end{equation}
and 
\begin{equation}
\label{sum_of_sigma_xi}
4\sum_{t\in [T]}\sigma_t^{\xi}(x_t)^2 \leq C\left(\frac{\eta^2 + \rho^2\eta_{\mathrm{ML}}^2}{1-\rho^2}\right)\,\gamma_T\left(k,\frac{\eta^2 + \rho^2\eta_{\mathrm{ML}}^2}{1-\rho^2}\right) \overset{(b)}{\leq} C\left(\frac{\eta^2 + \rho^2\eta_{\mathrm{ML}}^2}{1-\rho^2}\right)\,{\gamma_T(k,\eta^2)}.
\end{equation}
where inequalities $(a)$ and $(b)$ are true because $\frac{\eta^2 + \rho^2\eta_{\mathrm{ML}}^2}{1-\rho^2} \geq \eta^2$ and that  $\gamma_T(k,\sigma^2)$ is nonincreasing in $\sigma$ when fixing $k$, and $C(\sigma^2)$ is increasing in $\sigma^2$.

Substituting (\ref{sum_of_sigma_f}) and (\ref{sum_of_sigma_xi}) into  (\ref{bound2}), and utilizing (\ref{bound}) gives that,
\begin{align*}
    4\sum_{t\in[T]}(\sigma_t^{\text{PA}}(x_t))^2 &\overset{{(\ref{bound})}}{\leq}    \sum_{t\in[T]}
\left[1-(1-R)\rho_t^2(x_t)\right](\sigma_t^{\mathrm{true}}(x_t))^2\\
&\overset{(\ref{bound2})}{\leq} \sum_{t\in[T]} (1-\rho^2)(1-R)(\sigma_t^{\xi}(x_t))^2 + R(\sigma_t^{\mathrm{true}}(x_t))^2\\
&\overset{(\ref{sum_of_sigma_f},\ref{sum_of_sigma_xi})}{\leq} 
    \underbrace{\frac{ 8}{\log \left(1+\left({(\eta^2 + \rho^2\eta_{\mathrm{ML}}^2)}/{(1-\rho^2)}\right)^{-1}\right)}}_{\triangleq C_1} \left[(1-(1-R)\rho^2 \right] \gamma_T.
\end{align*}
Hence, we have
\[
R_T \leq \sqrt{C_1  \beta_T T\left[(1-(1-R)\rho^2 \right] \gamma_T} +\frac{\pi^2}{6}.
\]

\end{proof}

\section{Useful Formulas}
\label{app:formulas}

Throughout, let $k(\cdot,\cdot)\equiv K(\cdot,\cdot)$ be the prior scalar kernel and
$B=\begin{bmatrix}1&\rho\\ \rho&1\end{bmatrix}$ the covariance matrix. The
bivariate kernel is
\[
K\big((x,i),(x',j)\big)=B_{ij}\,k(x,x'),\qquad i,j\in\{\mathrm{true},\mathrm{ML}\}.
\]
For any set of inputs $X=\{x_1,\dots,x_n\}$ define the Gram matrix
$K_{XX}\in\mathbb{R}^{n\times n}$ by $(K_{XX})_{ab}=k(x_a,x_b)$ and the
cross-kernel row $k_{xX}\in\mathbb{R}^{1\times n}$ by $(k_{xX})_a=k(x,x_a)$.

\vspace{4pt}
\paragraph{Online paired observations.}
At online round $t$, let $X_{t-1}=\{x_1,\dots,x_{t-1}\}$ and stack the paired
observations as
\[
y_{t-1}=
\begin{bmatrix}
y_{t-1}\\
y^{\mathrm{ML}}_{t-1}
\end{bmatrix}
=
\begin{bmatrix}
y(x_1)\\ \vdots\\ y(x_{t-1})\\
y^{\mathrm{ML}}(x_1)\\ \vdots\\ y^{\mathrm{ML}}(x_{t-1})
\end{bmatrix}
\in\mathbb{R}^{2(t-1)}.
\]
The corresponding noise covariance is
\[
\Sigma_{\text{on}}=I_{t-1}\otimes
\begin{bmatrix}
\eta^2&0\\0&\eta_{\mathrm{ML}}^2
\end{bmatrix}.
\]
Define the multi-task Gram matrix
\[
K_{t-1}^{\text{on}} \triangleq K_{X_{t-1}X_{t-1}}\otimes B
\in\mathbb{R}^{2(t-1)\times 2(t-1)}.
\]
For a test point $x$, define the cross-covariance
\[
K_{x}^{\text{on}} \triangleq k_{xX_{t-1}}\otimes B
\in\mathbb{R}^{2\times 2(t-1)},
\qquad
K_{xx}\triangleq k(x,x)\,B\in\mathbb{R}^{2\times 2}.
\]
Then the online-only posterior (used for $\mu_t^{\mathrm{true}},\mu_t^{\mathrm{ML}},\sigma_t^{\mathrm{true}},\sigma_t^{\mathrm{ML}}$
and $\rho_t$) is
\begin{align*}
\mu_t^{\text{on}}(x)
&=
\mathbb{E}\!\left[
\begin{bmatrix}f(x)\\ f^{\mathrm{ML}}(x)\end{bmatrix}\Bigm|\mathcal{D}_{t-1}
\right]
=
K_{x}^{\text{on}}
\left(K_{t-1}^{\text{on}}+\Sigma_{\text{on}}\right)^{-1}
y_{t-1},
\\
\boldsymbol{\Sigma}_t^{\text{on}}(x)
&=
\mathrm{Var}\!\left(
\begin{bmatrix}f(x)\\ f^{\mathrm{ML}}(x)\end{bmatrix}\Bigm|\mathcal{D}_{t-1}
\right)
=
K_{xx}
-
K_{x}^{\text{on}}
\left(K_{t-1}^{\text{on}}+\Sigma_{\text{on}}\right)^{-1}
\left(K_{x}^{\text{on}}\right)^{\!\top}.
\end{align*}
We extract
\[
\mu_t^{\mathrm{true}}(x)=[\mu_t^{\text{on}}(x)]_1,\quad
\mu_t^{\mathrm{ML}}(x)=[\mu_t^{\text{on}}(x)]_2,\quad
(\sigma_t^{\mathrm{true}}(x))^2=[\boldsymbol{\Sigma}_t^{\text{on}}(x)]_{11},\quad
(\sigma_t^{\mathrm{ML}}(x))^2=[\boldsymbol{\Sigma}_t^{\text{on}}(x)]_{22},
\]
and the posterior correlation
\[
\rho_t(x)=
\frac{[\boldsymbol{\Sigma}_t^{\text{on}}(x)]_{12}}
{\sigma_t^{\mathrm{true}}(x)\sigma_t^{\mathrm{ML}}(x)}.
\]

\vspace{6pt}
\paragraph{Offline ML-prediction-only observations.}
Let $X^{\text{off}}=\{x^{\text{off}}_1,\dots,x^{\text{off}}_{n_{\text{off}}}\}$
and $y^{\text{off}}\in\mathbb{R}^{n_{\text{off}}}$ collect
$y^{\mathrm{ML}}(x^{\text{off}}_i)$.
Define
\[
K_{\text{off}}\triangleq K_{X^{\text{off}}X^{\text{off}}}\in\mathbb{R}^{n_{\text{off}}\times n_{\text{off}}},
\qquad
k_{x,\text{off}}\triangleq k_{xX^{\text{off}}}\in\mathbb{R}^{1\times n_{\text{off}}}.
\]
Since the prediction marginal has covariance $k(\cdot,\cdot)$, the offline
posterior for $f^{\mathrm{ML}}$ is the standard GP regression update:
\begin{align*}
\mu^{\text{off}}(x)
&\triangleq\mathbb{E}[f^{\mathrm{ML}}(x)\mid \mathcal{D}^{\text{off}}]
=
k_{x,\text{off}}
\left(K_{\text{off}}+\eta_{\mathrm{ML}}^2 I_{n_{\text{off}}}\right)^{-1}
y^{\text{off}},
\\
(\sigma^{\text{off}}(x))^2
&\triangleq\mathrm{Var}(f^{\mathrm{ML}}(x)\mid \mathcal{D}^{\text{off}})
=
k(x,x)
-
k_{x,\text{off}}
\left(K_{\text{off}}+\eta_{\mathrm{ML}}^2 I_{n_{\text{off}}}\right)^{-1}
k_{x,\text{off}}^{\top}.
\end{align*}
(If the offline stage averages $N$ duplicate prediction queries at each grid
point, replace $\eta_{\mathrm{ML}}^2$ by $\eta_{\mathrm{ML}}^2/N$ in the above.)

\vspace{6pt}
\paragraph{Augmented posterior using offline$+$online data.}
For the augmented model $\mathcal{GP}_{\mathrm{true},\mathrm{ML}_{\mathrm{all}}}$ conditioned on
$\mathcal{D}^{\text{off}}\cup\mathcal{D}_{t-1}$, stack the observations as
\[
y_{t-1}^{\text{all}}=
\begin{bmatrix}
y^{\text{off}}\\
y_{t-1}\\
y^{\mathrm{ML}}_{t-1}
\end{bmatrix}
\in\mathbb{R}^{n_{\text{off}}+2(t-1)}.
\]
The joint covariance of $y_{t-1}^{\text{all}}$ (before adding noise)
is the block matrix
\[
K_{t-1}^{\text{all}}=
\begin{bmatrix}
K_{\text{off}} & \rho K_{\text{off,on}} & K_{\text{off,on}}\\
\rho K_{\text{on,off}} & K_{\text{on}}\; & \rho K_{\text{on}}\\
K_{\text{on,off}} & \rho K_{\text{on}} & K_{\text{on}}
\end{bmatrix},
\]
where $K_{\text{on}}\triangleq K_{X_{t-1}X_{t-1}}$ and
$K_{\text{off,on}}\triangleq K_{X^{\text{off}}X_{t-1}}$ (so $K_{\text{on,off}}=K_{\text{off,on}}^{\top}$).
The corresponding noise covariance is
\[
\Sigma_{\text{all}}
=
\mathrm{diag}\!\left(\eta_{\mathrm{ML}}^2 I_{n_{\text{off}}},\, \eta^2 I_{t-1},\, \eta_{\mathrm{ML}}^2 I_{t-1}\right).
\]
For a test point $x$, the cross-covariance between
$\big(f(x),f^{\mathrm{ML}}(x)\big)$ and $y_{t-1}^{\text{all}}$ is
\[
K_{x}^{\text{all}}
=
\begin{bmatrix}
\rho\,k_{x,\text{off}} & k_{xX_{t-1}} & \rho\,k_{xX_{t-1}}\\
k_{x,\text{off}} & \rho\,k_{xX_{t-1}} & k_{xX_{t-1}}
\end{bmatrix}
\in\mathbb{R}^{2\times (n_{\text{off}}+2(t-1))}.
\]
Thus the augmented posterior is
\begin{align*}
\mu_t^{\text{all}}(x)
&=
K_{x}^{\text{all}}
\left(K_{t-1}^{\text{all}}+\Sigma_{\text{all}}\right)^{-1}
y_{t-1}^{\text{all}},
\\
\boldsymbol{\Sigma}_t^{\text{all}}(x)
&=
K_{xx}
-
K_{x}^{\text{all}}
\left(K_{t-1}^{\text{all}}+\Sigma_{\text{all}}\right)^{-1}
\left(K_{x}^{\text{all}}\right)^{\!\top}.
\end{align*}
We extract the prediction quantities used in PA-GP-UCB as
\[
\mu_t^{\mathrm{ML},\text{all}}(x)=[\mu_t^{\text{all}}(x)]_2,
\qquad
(\sigma_t^{\mathrm{ML},\text{all}}(x))^2=[\boldsymbol{\Sigma}_t^{\text{all}}(x)]_{22}.
\]

\section{Additional Numerical Analysis}
\subsection{Posterior Visualizations}
\label{addpost}
\begin{figure*}[t]
  \centering
  \vspace{0.5em}
\begin{overpic}[width=0.86\linewidth]{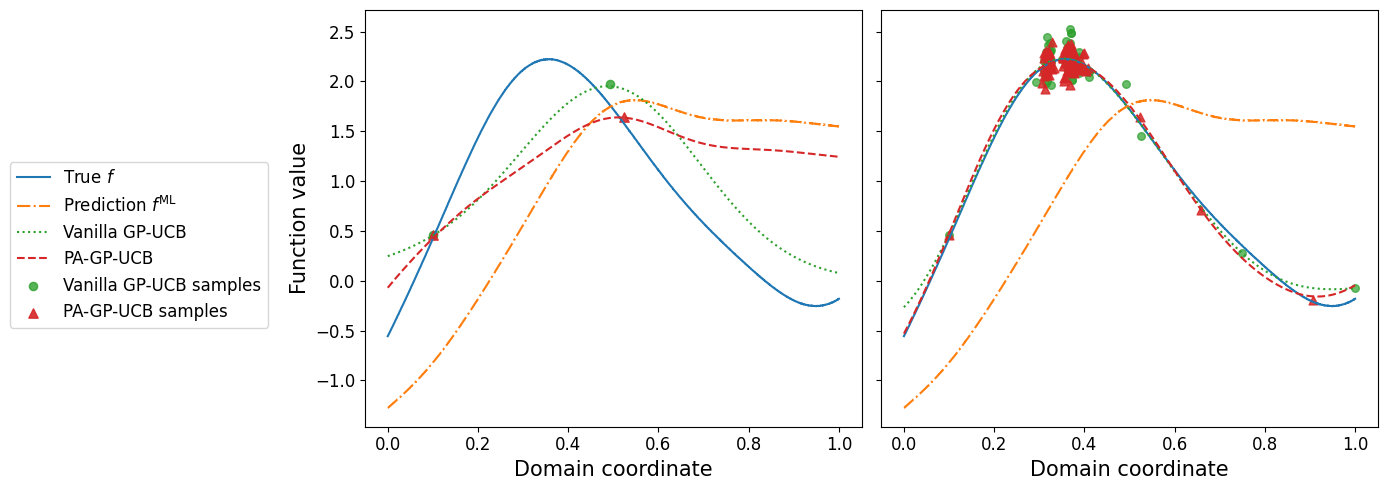}
    \put(43,36){\small\textit{(a)}}
    \put(80,36){\small\textit{(b)}}
\end{overpic}

\caption{
(a) Posterior means of Vanilla GP-UCB and PA-GP-UCB at 
$T = 2$ compared with the true function and the predicted function under 
correlation $\rho = 0.8$. 
(b) The corresponding posterior means at $T = 200$.
}

\label{fig:exp11}

\end{figure*}

Figure~\ref{fig:exp11} illustrates the evolution of the posterior under PA-GP-UCB and
Vanilla GP-UCB.
At $T=1$, both methods are initialized at a random location.
At $T=2$, PA-GP-UCB exploits information from the offline prediction
$f^{\mathrm{ML}}$, and its posterior closely resembles the shape of the predicted
function, placing its first query near the predicted optimum, even though the
prediction's optimum differs from that of the true objective.
As more human feedback is incorporated, the multi-task correction progressively
adjusts the posterior toward the true function.
By $T=200$, the posterior under PA-GP-UCB closely matches the ground-truth function
and concentrates around the true optimum, yielding an accurate reconstruction of the ground-truth function.
In contrast, Vanilla GP-UCB relies solely on sparse human feedback, resulting in
limited early guidance and slower convergence.

\subsection{Ablations on Noise Robustness}
\label{moreablations}

Figure~\ref{fig:exp122} reports additional ablations on observation noise
and prediction noise. Figure~\ref{fig:exp122}(a) highlights robustness to observation noise in the true reward function.
While reducing prediction noise improves performance overall, the regret gap
$R_T^{\mathrm{Vanilla}} - R_T$ shows that PA-GP-UCB yields larger gains
in high-noise regimes, where reliance on expensive human feedback alone is particularly
inefficient. Figure~\ref{fig:exp122}(b) varies the prediction noise variance $\eta_{\mathrm{ML}}^2$
with $M = N = 1000$.
PA-GP-UCB consistently outperforms Vanilla GP-UCB across all noise levels, achieving
larger improvements as prediction noise decreases, while remaining robust even under
highly noisy predictions.

\begin{figure*}[t]
  \centering
\begin{subfigure}{0.32\linewidth}
  \centering
  \caption{}
  \includegraphics[width=\linewidth]{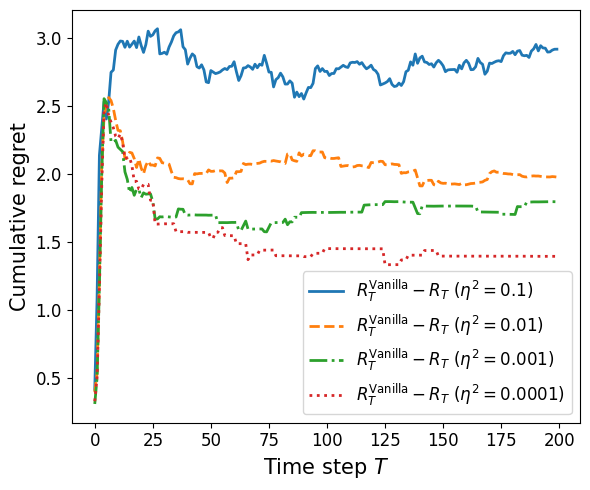}
\end{subfigure}
\hspace{0.04\linewidth}
\begin{subfigure}{0.32\linewidth}
  \centering
  \caption{}
  \includegraphics[width=\linewidth]{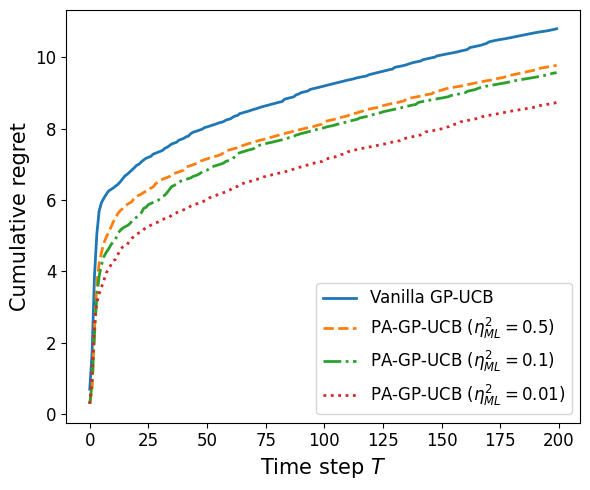}
\end{subfigure}

  \caption{
  Cumulative regret of Vanilla GP-UCB and PA-GP-UCB averaged over $50$ runs with horizon $T=200$.
  (a) Regret gap $R_T^{\mathrm{Vanilla}}-R_T$ for
$\eta^2\in\{0.1,0.01,0.001,0.0001\}$ with $M=N=1000$ and
$\eta_{\mathrm{ML}}^2=0.01$. (b) Effect of prediction noise, varying $\eta_{\mathrm{ML}}^2 \in \{0.5, 0.1, 0.01\}$,
with $M = N = 1000$ and $\eta^2 = 0.01$. }
  \label{fig:exp122}
\end{figure*}

\section{LLM prediction prompting templates}
\label{app:llm-prompts}

This appendix lists the exact prompt templates used to query the LLM 
for predicting average weekly visits. In all cases, we request a single JSON
number and use one deterministic query per arm. 
\subsection{$K$-shot calibrated predictions ( $K=4$ or $K=12$)}
\label{app:prompt-kshot}

\begin{verbatim}
You are predicting average weekly visits.

Here are example experimental conditions with their observed average visits:

1) "{cond_1}" -> {y_1:.2f}
2) "{cond_2}" -> {y_2:.2f}
...
K) "{cond_K}" -> {y_K:.2f}

Now predict the expected average weekly visits for:

"{arm_text}"

Return ONLY JSON:
{"pred_visits": number}
\end{verbatim}

\subsection{Scale-only predictions (min/max/mean only)}
\label{app:prompt-scale-only}

\begin{verbatim}
You are predicting average weekly visits.

Here is global context about the experiment:

- Visits are non-negative real numbers.
- Across this study, values typically range from {y_min:.1f} to {y_max:.1f}.
- The average condition receives around {y_mean:.1f} visits.
- Most conditions fall near the middle of this range.

Now predict the expected average weekly visits for:

"{arm_text}"

Return ONLY JSON:
{"pred_visits": number}
\end{verbatim}

\subsection{Notes on filling the template}
\label{app:prompt-notes}

In the scale-only prompt, $\{y_{\min}, y_{\max}, y_{\mathrm{mean}}\}$ are computed from the set of
ground-truth arm-level outcomes used for calibration and evaluation, and
$\texttt{arm\_text}$ is the natural-language description of the arm. In the $K$-shot
prompt, the pairs $\{( \texttt{cond}_j, y_j )\}_{j=1}^K$ are sampled from the available
arm-level ground truth, with a fixed random seed for reproducibility.


\end{document}